\documentclass{article}

\usepackage{microtype}
\usepackage{graphicx}
\usepackage{subfigure}
\usepackage{booktabs} 
\usepackage{amssymb}
\usepackage{graphicx}
\usepackage{color}
\usepackage{amsmath, amssymb, amstext, amsthm}
\usepackage{mathtools}
\usepackage{graphicx}
\usepackage{thmtools}
\usepackage{thm-restate}
\usepackage{tabularx}
\usepackage{comment} 
\usepackage{algorithm}
\usepackage{algorithmic}
\usepackage{float}
\usepackage{wrapfig}
\usepackage{mathtools}
\usepackage{thmtools,thm-restate}
\usepackage{multirow}
\usepackage{enumitem}
\usepackage[noorphans,vskip=0ex]{quoting}
\usepackage{xcolor}
\usepackage{xspace}
\usepackage{hyperref} 
\usepackage{cleveref}
\usepackage{custom_commands}
\usepackage{hz_commands}
\usepackage{hyperref}
\usepackage[T1]{fontenc}
\usepackage{lmodern}

\usepackage[accepted]{icml2021}

\icmltitlerunning{Bridging Multi-Task Learning and Meta-Learning}

\begin{document}

\twocolumn[
\icmltitle{Bridging Multi-Task Learning and Meta-Learning:\\
Towards Efficient Training and Effective Adaptation}



\icmlsetsymbol{equal}{*}

\begin{icmlauthorlist}
\icmlauthor{Haoxiang Wang}{uiuc}
\icmlauthor{Han Zhao}{uiuc}
\icmlauthor{Bo Li}{uiuc}

\icmlaffiliation{uiuc}{University of Illinois at Urbana-Champaign, Urbana, IL, USA}
\end{icmlauthorlist}
\icmlcorrespondingauthor{Haoxiang Wang}{hwang264@illinois.edu}
\icmlcorrespondingauthor{Han Zhao}{hanzhao@illinois.edu}
\icmlcorrespondingauthor{Bo Li}{lbo@illinois.edu}

\icmlkeywords{Machine Learning, ICML}

\vskip 0.3in
]



\printAffiliationsAndNotice{}  
\begin{abstract}
Multi-task learning (MTL) aims to improve the generalization of several related tasks by learning them jointly. As a comparison, in addition to the joint training scheme, modern meta-learning allows unseen tasks with limited labels during the test phase, in the hope of fast adaptation over them. Despite the subtle difference between MTL and meta-learning in the problem formulation, both learning paradigms share the same insight that the shared structure between existing training tasks could lead to better generalization and adaptation. In this paper, we take one important step further to understand the close connection between these two learning paradigms, through both theoretical analysis and empirical investigation. Theoretically, we first demonstrate that MTL shares the same optimization formulation with a class of gradient-based meta-learning (GBML) algorithms. We then prove that for over-parameterized neural networks with sufficient depth, the learned predictive functions of MTL and GBML are close. In particular, this result implies that the predictions given by these two models are similar over the same unseen task. Empirically, we corroborate our theoretical findings by showing that, with proper implementation, MTL is competitive against state-of-the-art GBML algorithms on a set of few-shot image classification benchmarks. Since existing GBML algorithms often involve costly second-order bi-level optimization, our first-order MTL method is an order of magnitude faster on large-scale datasets such as mini-ImageNet. We believe this work could help bridge the gap between these two learning paradigms, and provide a computationally efficient alternative to GBML that also supports fast task adaptation. 
\end{abstract}

\section{Introduction}\label{sec:intro}
Multi-task learning has demonstrated its efficiency and effectiveness on learning shared representations with training data from multiple related tasks simultaneously~\cite{caruana1997multitask,ruder2017overview,overview-mtl}. Such shared representations could transfer to many real-world applications, such as object detection~\cite{zhang2014facial}, image segmentation~\cite{kendall2018multi}, multi-lingual machine translation~\cite{dong2015multi}, and language understanding evaluation~\cite{wang2018glue}.
On the other hand, in addition to the joint training scheme, modern meta-learning can leverage the shared representation to fast adapt to unseen tasks with only minimum limited data during the test phase~\cite{hospedales2020metalearning}. As a result, meta-learning has drawn increasing attention and been applied to a wide range of learning tasks, including few-shot learning~\cite{snell2017prototypical,matching-net,metaOptNet}, meta reinforcement learning~\cite{maml}, speech recognition~\cite{hsu2020meta} and bioinformatics~\cite{luo2019mitigating}.

Despite their subtle differences in problem formulation and objectives, both MTL and meta-learning aim to leverage the correlation between different tasks to enable better generalization to either seen or unseen tasks. However, a rigorous exploration of this intuitive observation is severely lacking in the literature. As a result, while being effective on fast adaptation to unseen tasks, many meta-learning algorithms still suffer from expensive computational costs~\citep{nichol2018first,howtotrainmaml,hospedales2020metalearning}. On the other hand, while being efficient in training, due to its problem formulation, MTL does not allow adaptation to unseen tasks, at least in a straightforward manner. Hence, a natural question to ask is, 
\begin{quoting}
\itshape
\vspace*{-0.2em}
    Can we combine the best of both worlds from MTL and meta-learning, i.e., fast adaptation to unseen tasks with efficient training? 
\vspace*{-0.1em}
\end{quoting}
To answer this question, one needs to first understand the relationship between MTL and meta-learning in greater depth. To this end, in this paper, we take the \emph{first} attempt with the goal to bridge these two learning paradigms. In particular, we focus on a popular class of meta-learning methods, gradient-based meta-learning (GBML), which takes a bi-level optimization formulation inspired from the Model-Agnostic Meta-Learning (MAML)~\cite{maml}. From an optimization perspective, we first show that MTL and a class of GBML algorithms share the same optimization formulation. Inspired by this simple observation, we then prove that, for sufficiently wide neural networks, these two methods lead to close predictors: on any test task, the predictions given by these two methods are similar, and the gap is inversely proportional to the network depth. Our theoretical results imply that, in principle, it is possible to improve the existing MTL methods to allow fast adaptation to unseen tasks, without loss in its training efficiency, and thus provide an affirmative answer to the above question.

Empirically, to corroborate our findings, we first conduct a series of experiments on synthetic data to show the increasing closeness between the predictors given by MTL and GBML, as the network depth grows. We then perform extensive experiments to show that with proper implementation, MTL can achieve similar or even better results than the \emph{state-of-the-art} GBML algorithms, while enjoys \emph{significantly lower} computational costs. This indicates that MTL could be potentially applied as a powerful and efficient alternative to GBML for meta-learning applications. 

Our contributions could be briefly summarized as follows:

\begin{itemize}[leftmargin=*,align=left,noitemsep,nolistsep]
    \item \emph{Bridging MTL and GBML from the optimization perspective:} We show that MTL and a class of GBML algorithms share the same {optimization formulation}. In particular, GBML takes a {regularized bi-level optimization} while MTL adopts the simple {joint training}.
    \item \emph{Closeness in the function space:} 
    we prove that for over-parameterized neural nets with sufficient width, the learned predictive functions of MTL and a GBML algorithm are close in the function space, indicating their predictions are similar on unseen (test) tasks. Furthermore, we empirically validate this theoretical result on synthetic data.
    \item \emph{Empirical performance and efficiency:} Motivated by our theoretical results, we implement MTL with modern deep neural nets, and show that the performance of MTL is competitive against MetaOptNet \cite{metaOptNet}, a \textit{state-of-the-art} GBML algorithm, on few-shot image classification benchmarks. Notably, the training of MTL is \textit{an order of magnitude faster} than that of MetaOptNet, due to its first-order optimization. The code is released at \url{https://github.com/AI-secure/multi-task-learning}
\end{itemize}

\section{Related Work}\label{sec:related-works}
\textbf{Multi-Task Learning}~~Multi-task learning (MTL) is a method to jointly learn shared representations from multiple training tasks~\cite{caruana1997multitask}. Past research on MTL is abundant. Theoretical results on learning shared representation with MTL have shown that the joint training scheme is more sample efficient than single-task learning, at least under certain assumptions of task relatedness, linear features and model classes~\cite{maurer2016benefit,tripuraneni2020theory}. Other works on MTL include designing more efficient optimization methods to explore the task and feature relationships~\citep{evgeniou2007multi,argyriou2008convex,zhang2010convex,zhao2020efficient}. 

\textbf{Meta-Learning}~~Meta-learning, or learning-to-learn, is originally proposed for few-shot learning tasks~\cite{learningtolearn,baxter1998theoretical}, where the goal is fast adaptation to unseen tasks. Among various meta-learning methods, a line of works following MAML, termed as gradient-based meta-learning (GBML)~\citep{maml,imaml}, has been increasingly applied in many downstream application domains. Recent works on understanding GBML have shown that MAML is implicitly performing representation learning, which is the key to its empirical success~\citep{raghu2019rapid}. In particular, \citet{saunshi2020sample} compares MTL and Reptile \cite{reptile}, a first-order variant of MAML, in a toy setting of \textit{1d} subspace learning with \textit{2-layer} \textit{linear} models, and shows the upper bounds of their sample complexity are of the same order. In contrast, our theory is compatible with \textit{non-linear} neural nets of \textit{any depth} and has no restriction on the input dimension, which is a more realistic and practical setting. In addition to GBML, the considered MTL implementation shares some similarities with metric-based meta-learning (i.e., metric learning) methods in few-shot learning scenarios \citep{snell2017prototypical,matching-net}, since here we also only keep the trained hidden layers for test.

\section{Preliminaries}\label{sec:prelim}
We first provide a brief discussion to the common network architectures, training algorithms, and evaluation protocols for MTL and GBML.

\subsection{Neural Networks Architectures}
Consider a $L$-layer fully-connected neural network $f_\theta \colon \bR^{d} \mapsto \bR^{k}$, which contains $\width_{l}$ neurons in the $l$-th hidden layer for $l\in[L-1]$. Denote the parameters of the first $L-1$ layers (i.e., hidden layers) as $\theta^{\body}$, and the last hidden layer output (i.e., network representation/features) as $\phitop_{\theta^\body}\colon \bR^{d}\mapsto \bR^{\width_{L-1}}$. For simplicity, we assume the output layer (i.e., network \textit{head}) has no bias, and denote it as $w \in \bR^{\width_{L-1}\times k}$. Thus, for any input $x\in \bR^{d}$, the neural network output can be expressed as $$f_\theta(x) = \phitop_{\theta^\body}(x) ~ w ~.$$
Networks used in MTL often have a \emph{multi-head} structure, where each head corresponds to a training task. In this case, the shared hidden layers are treated as the shared representations. Formally, denote a $L$-layer $N$-head neural network as $\fmtl_\thetamtl \colon \bR^{d} \times [N] \mapsto \bR^{k}$, s.t. for any input $x\in \bR^{d}$ and \textit{head index} $i\in [N]$, the network output is
\begin{align}
    \fmtl_{\thetamtl}(x,i) =\phitop_{\thetamtl^\body}(x) ~ \wmtl{i},
\end{align}
where $\phi_{\thetamtl^\body}(x)$ is the last hidden layer output, $\thetamtl^\body$ is the parameters of first $L-1$ layers, and $w_i$ is the $i$-th head in the output layer. Note that the network parameters are the union of parameters in the hidden layers and multi-head output layer, i.e., $\thetamtl = \{\thetamtl^\body\}\cup \{\wmtl{i}\}_{i\in[N]}$.

\vspace{-0.5em}
\subsection{Multi-Task Learning}\label{sec:prelim:mtl}
In MTL, a multi-head neural net with $N$ heads is trained over $N$ training tasks each with $n$ samples \cite{ruder2017overview}. For $i\in[N]$, denote the data for the $i$-th task as $(X_i,Y_i)$, where $X_i\in \bR^{n\times d}$ and $Y_i \in \bR^{n\times k}$. The training of MTL is to minimize the following objective given loss function $\ell$,
\begin{align}\label{eq:prelim:mtl-loss}
    \min_{\thetamtl} \loss_\mtl(\thetamtl) \coloneqq
    \sum_{i\in [N]} \ell \left(\phitop_{\thetamtl^\body}(X_i)~\wmtl{i}, ~Y_i\right)
\end{align}
where $\phitop_{\thetamtl^\body}(X_i) \coloneqq \big(\phitop_{\thetamtl}(x)\big)_{x\in[X_i]}\in \bR^{n\times \width_{L-1}}$.
\subsection{Gradient-Based Meta-Learning and ANIL}
\label{sec:prelim:gbml}
Here we introduce a representative algorithm of GBML, \textit{Almost-No-Inner-Loop} (ANIL) \cite{raghu2019rapid}, which is a simplification of MAML. The setup is the same as Sec. \ref{sec:prelim:mtl}, where $N$ {training tasks} each with $n$ sample-label pairs are provided, i.e., $\{X_i,Y_i\}_{i=1}^N$. In practice, a training protocol, \textit{query-support split} (cf.\ Appendix $\ref{supp:background}$ for more details), is often adopted. However, recent work has shown that such a split is not necessary~\cite{bai2021how}. Hence, we do not consider the query-support split through this work.

ANIL minimizes the following loss over $\theta=\{\theta^\body,w\}$,
\begin{align}
    &\min_\theta \loss_\anil(\theta) \coloneqq \sum_{i\in[N]} \ell(\phitop_{\theta^\body}(X_i) ~ w_i^*, ~Y_i) \label{eq:prelim:anil-no-split:outer-loop}\\
    &\text{s.t. }w_i^*=~\innerloop(w, \phitop_{\theta^\body}(X_i), Y_i, \tau,\lambda)\label{eq:prelim:anil-no-split:inner-loop}
\end{align} 
where \eqref{eq:prelim:anil-no-split:inner-loop} is the common \textit{inner-loop} optimization of GBML, which \textit{runs $\tau$ steps of gradient descent w.r.t. $w$ on the loss $\ell(\phitop_{\theta^\body}(X_i) w, Y_i)$, with learning rate $\lambda$}. 

Notably, \citet{lin2021to} empirically shows that with a frozen head $w$ across training, ANIL has no performance drop, indicating that optimizing over $w$ in the outer loop \eqref{eq:prelim:anil-no-split:outer-loop} is insignificant. Thus, the corresponding training objective of this  ANIL without the outer-loop optimization of $w$ is
\begin{align}
    &\min_{\theta^\body} \loss_\anil(\theta) \coloneqq \sum_{i\in[N]} \ell(\phitop_{\theta^\body}(X_i) ~ w_i^*, ~Y_i) \label{eq:prelim:anil-frozen-head:outer-loop}
\end{align}
with $w_i^*$ defined in the same way as \eqref{eq:prelim:anil-no-split:inner-loop}.

\vspace{-0.5em}
\subsection{Fine-Tuning for Test Task Adaptation}\label{sec:prelim:fine-tune}
In the test phase of few-shot learning, an arbitrary test task $\task$ consists of
$(\xyxy) \in \bR^{n\times d}\times  \bR^{n\times k}\times \bR^{n'\times d}\times \bR^{n'\times k}$,
\label{eq:prelim:def-test-task}
where $(X,Y)$ are \textit{query} data and $(X',Y')$ are \textit{support} data. Note that the original formulation of MTL with multi-head network structures does not support adaptation to unseen tasks. To compare MTL and GBML on an equal footing, in this work, we adopt the following same test protocol on both MTL and GBML. First, a randomly initialized head $\wtest$ is appended to the last hidden layer of networks trained under MTL or GBML. Then, the head $\wtest$ is \textit{fine-tuned} on labelled \textit{support} samples $(X',Y')$, and the network makes predictions on the \textit{query} samples $X$.
Specifically, for a trained MTL model with parameters $\thetamtl$, its prediction on $X$ after fine-tuning $\wtest$ on $(X',Y')$ for $\hat\tau$ steps is
\begin{align}
    &F_\mtl(\xxy) = \phitop_{\thetamtl^\body}(X) ~\wtest^*\label{eq:fine-tune:mtl:1}\\
    &\text{s.t. }\wtest^* = \innerloop(\wtest, \phitop_{\thetamtl^\body}(X'),Y',\hat{\tau},\lambda)\label{eq:fine-tune:mtl:2}
\end{align}
where $\wtest^*$ is the fined-tuned test head after $\hat\tau$ steps of gradient descent on $\wtest$, and $\innerloop$ is defined in the same way as \eqref{eq:prelim:anil-no-split:inner-loop}. Similarly, the prediction of a trained ANIL model with parameters $\theta$ on $X$ is
\begin{align}
    &F_\anil(\xxy) = \phitop_{\theta^\body}(X) ~\wtest^*\label{eq:fine-tune:anil:1}\\
    &\text{s.t. }\wtest^* = \innerloop(\wtest, \phitop_{\theta^\body}(X'),Y',\hat{\tau},\lambda)\label{eq:fine-tune:anil:2}
\end{align}

\vspace{-1em}
\section{Theoretical Analysis}
\label{sec:theory}
In this section, we compare MTL with a class of GBML algorithms, and show that \textit{(i)} essentially, they optimize the same objective with different optimization approaches, \textit{(ii)} the learned predictors from both algorithms are close under a certain norm in the function space. 

\vspace{-1em}
\subsection{A Taxonomy of Gradient-Based Meta-Learning}
Various GBML methods often differ in the way on how to design the \emph{inner-loop optimization} in Eq.~\eqref{eq:prelim:anil-no-split:inner-loop}. For example, MAML, ANIL and some other MAML variants usually take \textit{a few} gradient descent steps (typically $1$$\sim$$10$ steps), which is treated as an \textit{early stopping} type of regularization~\cite{imaml,grant2018recasting}. As a comparison, another line of GBML algorithms uses the explicit \textit{$\ell_2$ regularization} in the inner loop instead~\cite{imaml,metaOptNet,r2d2,zhou2019efficient,goldblum2020unraveling}. In addition to regularization, variants of GBML methods also differ in the exact layers to optimize in the inner loop: While MAML optimizes all network layers in the inner loop, some other GBML algorithms are able to achieve state-of-the-art performance \cite{metaOptNet} by only optimizing the \textit{last layer} in the inner loop.

Based on the different \textit{regularization} strategies and \textit{optimized layers} in the inner-loop, we provide a taxonomy of GBML algorithms in Table~\ref{tab:taxonomy}.
\begin{table*}[tb]
    \centering
    \caption{A taxonomy of gradient-based meta-learning algorithms based on the algorithmic design of Eq.~\eqref{eq:prelim:anil-no-split:inner-loop}.}
    \label{tab:taxonomy}
    \vspace{+5pt}
    \begin{tabular}{l|*2l}\toprule
    Inner-Loop Optimized Layers &  \textbf{Early Stopping} & \textbf{$\ell_2$ Regularizer}\\\midrule
    \textbf{Last} Layer  &  \multirow{2}{*}{ANIL \citep{raghu2019rapid}} & MetaOptNet \citep{metaOptNet} \\
                         &                                               & R2D2 \citep{r2d2} \\\midrule
    \textbf{All} Layers  &  \multirow{2}{*}{MAML \citep{maml}} & iMAML \citep{imaml} \\
                        &                                            & Meta-MinibatchProx \citep{zhou2019efficient}\\\bottomrule
    \end{tabular}
\end{table*}

For algorithms that only optimize the last layer in the inner-loop, we formulate their training objectives in a \textit{unified framework}:
\begin{align}
    & \quad   \min_{\theta^\body} ~~ \sum_{i\in[N]} \ell \big(\phitop_{\theta^\body}(X_{i})~w_i^*, ~Y_{i}\big) \label{eq:unified-meta:outer}\\
    &   \text{s.t.} ~w_i^* = \argmin_{w_i} ~\ell \big(\phitop_{\theta^\body}(X_{i})~w_i, ~Y_{i}\big) + R(w_i)\label{eq:unified-meta:inner}
\end{align}
\begin{table}[tb]
\vspace*{-1em}
\centering
\caption{Typical instantiations of the problem~\eqref{eq:unified-meta:outer}. $\alpha \in \bR^+$ controls the strength of the regularization.}
\label{tab:instan}
\vspace{+5pt}
\begin{tabular}{*4l}
\toprule
              & ANIL & MetaOptNet         & R2D2               \\
\midrule
$\ell$        & Cross-Entropy & SVM Loss            & Squared Loss          \\
\midrule
$R(w_i)$ & Early Stopping & $\alpha\|w_i\|_2^2$ & $\alpha\|w_i\|_2^2$\\
\bottomrule
\end{tabular}
\end{table}

Certainly, there are abundant choices for the loss function $\ell$ and the regularization $R(w_i)$, and we summarize the typical choices used in the literature in Table~\ref{tab:instan}. For algorithms that optimize all layers in the inner loop, a unified framework similar to \eqref{eq:unified-meta:outer} and \eqref{eq:unified-meta:inner} is provided in Appendix \ref{supp:background}.

\subsection{Equivalence Between GBML and MTL from an Optimization Perspective}
In this section, we provide a simple observation that, surprisingly, the optimization objective of MTL shares the same formulation as that of GBML algorithms, \eqref{eq:unified-meta:outer}. Specifically, the objective function of MTL, \eqref{eq:prelim:mtl-loss}, can be re-written as
\begin{align}\label{eq:meta=mtl:mtl-obj}
   \min_{\thetamtl^\body, \{\wmtl{i}\}_{i=1}^N}
    \overbrace{\sum_{i\in [N]} \ell \left(\phitop_{\thetamtl^\body}(X_i)~\wmtl{i}, ~Y_i\right)}^{\loss_\mtl(\thetamtl)} 
\end{align}
where $\{\wmtl{i}\}_{i=1}^N$ are heads of a multi-head neural net, and $\thetamtl=\{\thetamtl^\body\}\cup\{\wmtl{i}\}_{i=1}^N$. As a comparison, if we plug \eqref{eq:unified-meta:inner} into \eqref{eq:unified-meta:outer}, the GBML objective \eqref{eq:unified-meta:outer} can be simplified as
\begin{align}\label{eq:meta=mtl:meta-obj}
    \min_{\theta^\body} \overbrace{\Big[\min_{\{w_i\}_{i=1}^N} \sum_{i\in[N]} \ell \Big(\phitop_{\theta^\body}(X_{i})~w_i, Y_{i}\Big) + R(w_i)\Big]}^{\loss_\gbml(\theta^\body)}
\end{align}
Note that different from~\eqref{eq:meta=mtl:mtl-obj}, the heads $\{w_i\}_{i=1}^N$ in~\eqref{eq:meta=mtl:meta-obj} are transient, in the sense that GBML algorithms do not explicitly save them during training. On the other hand, $\theta^\body$ contains all parameters to optimize in \eqref{eq:meta=mtl:meta-obj}, and $\theta^\body$ is optimized over $\ell_\gbml(\theta^\body)$, which is obtained by plugging in the minimizer of $\{w_i\}_{i=1}^N$ on the regularized loss. In other words, \eqref{eq:meta=mtl:meta-obj} is a bi-level optimization problem, with outer-loop optimization on network parameters $\theta^\body$ and inner-loop optimization on the transient heads $\{w_i\}_{i=1}^N$.

Clearly, up to the regularization term, the optimization problems \eqref{eq:meta=mtl:mtl-obj} and \eqref{eq:meta=mtl:meta-obj} share the same structure and formulation. In terms of the algorithms used to solve these two optimization problems, it is worth pointing out that GBML usually solves \eqref{eq:meta=mtl:meta-obj} as a \textit{bi-level} program where for each fixed $\theta^\body$, the algorithm will first compute the optimal heads $\{w_i\}_{i=1}^N$ as a function of $\theta^\body$, whereas in MTL,  \eqref{eq:meta=mtl:mtl-obj} is solved by the simple \textit{joint optimization} over both $\thetamtl^\body$ and $\{\wmtl{i}\}_{i=1}^N$.

From the discussions above, we conclude that the optimization formulation of GBML is equivalent to that of MTL, where the only difference lies in the optimization algorithms used to solve them. Motivated by this observation, in the next section, we explore the equivalence of these two algorithms in terms of the predictors obtained after convergence, when the networks are sufficiently wide. 

\subsection{Closeness Between MTL and GBML from a Functional Perspective}\label{sec:mtl=meta:functional}
In this section, we theoretically analyze MTL and a representative GBML algorithm, ANIL \cite{raghu2019rapid}, from a \textit{functional} perspective, and show that the learned predictors of MTL and ANIL after convergence are close under a certain norm. Due to the page limit, we defer detailed proofs to appendix, and mainly focus on discussing the implications of our theoretical results. Before we proceed, we first formally introduce the problem setup and training protocol used in the following analysis.

\textbf{Problem Setup}~~To simplify our analysis and presentation, we consider the squared loss, i.e., $\ell(\hat y,y) = \frac{1}{2}\|\hat y - y\|_2^2$. Note that the use of squared loss is standard for theoretical analyses of neural net optimization~\cite{ntk,du2019icml,AllenZhu2018ACT}. Furthermore, recently, \citet{hui2020evaluation} has also empirically demonstrated the effectiveness of squared loss in classification tasks from various domains. For the activation function and initialization scheme of neural nets, we focus on networks with ReLU activation and He's initialization\footnote{This is the common and default initialization scheme in Keras and PyTorch.}~\cite{resnet}, which is also standard in practice.

With the squared loss, the objectives of MTL and ANIL, i.e., \eqref{eq:prelim:mtl-loss} and \eqref{eq:prelim:anil-no-split:outer-loop}, can be simplifed to 
\begin{align*}
    \loss_\mtl(\thetamtl_t) & = \frac{1}{2} \sum_{i\in[N]}\left\|\vecop\left(\phitop_{\thetamtl^\body}(X_i)~\wmtl{i} - Y_i\right)\right\|_2^2, \\
    \loss_\anil(\theta_t) &= \frac{1}{2}\sum_{i \in [N]} \Big\|\vecop\Big(\phitop_{\theta^\body}(X_i) ~ w_i^* - Y_i\Big)\Big\|_2^2,
\end{align*}
where $\vecop(\cdot)$ is the vectorizaton operation and $w_i^*=\innerloop(w, \phitop_{\theta^\body}(X_i), Y_i, \tau,\lambda)$. During the test phase, for networks trained by MTL and ANIL, the predictions on any test task are obtained by fine-tuning an output head and predicting with this fine-tuned head (cf.\ Sec.~\ref{sec:prelim:fine-tune}). 

\textbf{Training Dynamics}~~We consider gradient flow (i.e., continuous-time gradient descent) for the training of both MTL and ANIL, which is a common setting used for theoretical analysis of neural nets with more than two layers~\citep{ntk,lee2019wide,CNTK}. In more detail, let $\eta$ be the learning rate, then the training dynamics of MTL and ANIL can be described by
\begin{align}\label{eq:grad-flow:mtl-meta}
    \frac{d \thetamtl_t}{dt} = -\eta \nabla_{\thetamtl_t} \loss_\mtl(\thetamtl_t), ~~\frac{d \theta_t}{dt} = -\eta \nabla_{\theta_t} \loss_\anil(\thetamtl_t)
\end{align}
where $\theta_t$ and $\thetamtl_t$ are network parameters at training step $t$.

\textbf{NTK and NNGP Kernels}~~Our forthcoming theoretical analysis involves both the Neural Tangent Kernel (NTK)~\cite{ntk,du2019icml,AllenZhu2018ACT} and the Neural Network Gaussian Process (NNGP) kernel~\citep{lee2018deep,novak2019bayesian}, which are tools used to understand the training trajectories of neural nets by reduction to classic kernel machines. For completeness, here we provide a brief introduction to both, so as to pave the way for our following presentation. Let the kernel functions of NTK and NNGP be $\NTK(\cdot,\cdot)$ and $\NNGP(\cdot,\cdot)$, respectively. Analytically, the NTK and NNGP for networks of $L$ layers can be computed recursively layer by layer~\citep{lee2019wide,CNTK}. Numerically, both kernels can be computed by using the Neural Tangents package \cite{neuraltangents2020}. Furthermore, without loss of generality, we assume the inputs are normalized to have unit variance, following~\citep{xiao2020dis}, and we adopt the NTK parameterization~\citep{lee2019wide}, which is the same with the standard neural net parameterization in terms of network output and training dynamics. More details about the parametrization can be found in Appendix~\ref{supp:background}.

With the above discussions clearly exposed, now we are ready to state the following main lemma, which serves as the basic for our main result in this section. In particular, by leveraging tools from \citet{lee2019wide} and \citet{meta-ntk}, we are able to prove that for sufficiently wide neural nets trained under gradient flow of ANIL or MTL (i.e., by \eqref{eq:grad-flow:mtl-meta}), their predictions on any test task are equivalent to a special class of kernel regression, with kernels that we name as \textit{(i)} ANIL Kernel $\metaNTK_\anil$ and \textit{(ii)} MTL Kernel $\metaNTK_\mtl$. Notice that both $\metaNTK_\anil$ and $\metaNTK_\mtl$ are composite kernels built on the NTK $\NTK$ and NNGP kernels $\NNGP$. 

\begin{lemma}[Test Predictions of MTL \& ANIL] \label{lemma:test-predict}
Consider an arbitrary test task $\task = (\xyxy)$, as defined in Sec. \ref{eq:prelim:def-test-task}. For arbitrarily small $\delta>0$, there exists $\eta^*,h^*\in \bR_+$ such that for networks with width greater than $h^*$ and trained under gradient flow with learning rate $\eta < \eta^*$, with probability at least $1-\delta$ over random initialization, the test predictions on $\task$ (i.e., Eq.~\eqref{eq:fine-tune:mtl:1} and~\eqref{eq:fine-tune:anil:1}) are 
\begin{align}
    &F_\mtl(\xxy) = G_{\hat\tau}(\xxy) \label{eq:lemma:test-pred:F_mtl}\\
     &\quad+\metaNTK_\mtl'((\xx,\hat\tau),\X) \metaNTK_\mtl^{-1}(\X,\X)\cdot \Y,\nonumber\\
     &F_\anil(\xxy) = G_{\hat\tau}(\xxy) \label{eq:lemma:test-pred:F_anil}\\
     &\quad+\metaNTK_\anil'((\xx,\hat\tau),\X) \metaNTK_\anil^{-1}(\X,\X)\big[\Y-G_\tau(\X,\X,\Y)\big],\nonumber
\end{align}
up to an error of $\cO(\frac{1}{\sqrt{h^*}})$ measured in $\ell_2$ norm. In above equations, we used shorthand $\X = (X_i)_{i=1}^N\in \bR^{Nn\times d}$ and $\Y=\vecop((Y_i)_{i=1}^N) \in \bR^{Nnk}$. Besides, the function $G$, kernels $\metaNTK_\mtl$ \& $\metaNTK_\anil$, and their variants $\metaNTK_\mtl'$ \& $\metaNTK_\anil'$, are defined below. 
    
\begin{itemize}[leftmargin=*,align=left]
    \item \textbf{Function $G$.} The function $G$ is defined as
    \begin{align}
     &\quad G_{\hat\tau}(\xxy) \nonumber\\
     &=\NNGP(X,X')  \NNGP(X',X')^{-1} (I-e^{-\lambda \NNGP(X',X') \hat\tau}) Y'\nonumber
    \end{align}
    and $G_{\tau}(\XXY) = \vecop((G_{\tau}(X_i,X_i,Y_i))_{i=1}^N)$.
    \item \textbf{MTL Kernels.} The kernel $\metaNTK_\mtl(\X,\X)$ is a block matrix of $N\times N$ blocks. Its $(i,j)$-th block for any $i,j\in[N]$ is
        \begin{align}
            [\metaNTK_\mtl(\X,\X)]_{ij} = \NTK (X_i,X_j) - \indicator [i\neq j] \NNGP(X_i,X_j)~. \nonumber
        \end{align}
        Besides, $\metaNTK_\mtl'$ is variant of the kernel function $\metaNTK_\mtl$, and $\metaNTK_\mtl'((X,X',\hat\tau),\X)$ is also a block matrix, of $1 \times N$ blocks, with the $(1,j)$-th block as
        \begin{align*}
            & [\metaNTK_\mtl'((X,X',\hat\tau),\X)]_{1j} =
            \NTK(X,X_j) - \NNGP(X,X_j)\nonumber \\ 
            &~~ - \NNGP(X,X')\T_{\NNGP}^{\hat\tau}(X') \Big[\NTK(X',X_j) - \NNGP(X',X_j)\Big] 
        \end{align*}
        where the function $T$ is defined as
        \begin{equation}\label{eq:T}
            \T_{\NNGP}^{\hat\tau}(X') = \NNGP(X',X')^{-1}\left(I - e^{-\lambda \NNGP(X')\hat\tau}\right)
        \end{equation}
    \item \textbf{ANIL kernels.} $\metaNTK_\anil(\X,\X)$ is also a block matrix of $N\times N$ blocks. Its $(i,j)$-th block for any $i,j\in[N]$ is
    \begin{align*}
        &\quad [\metaNTK_\mathrm{ANIL}(\X,\X)]_{ij} \nonumber\\
        &= e^{-\lambda \NNGP(X_i,X_i)\tau} \NTK(X_i,X_j) e^{-\lambda \NNGP(X_j,X_j)\tau},
    \end{align*}
    while $\metaNTK_\anil'((\xx,\hat\tau),\X)$ is a block matrix of $1 \times N$ blocks, with the $(1,j)$-th block as 
    \begin{align*}
        & [\metaNTK_\anil'((\xx,\hat\tau),\X)]_{1j} =\NTK(X,X_j)e^{-\lambda \NNGP(X_j,X_j) \tau}  \nonumber\\
        & - \NNGP(X,X') \T_{\NNGP}^{\hat\tau}(X') \NTK(X',X_j)e^{-\lambda \NNGP(X_j,X_j) \tau} 
    \end{align*}
    \end{itemize}
\end{lemma}
\textbf{Remark}~~The function $G$ is implicitly related to task adaptation. For instance, on the test task $\task=(\xyxy)$, $G_{\hat\tau}(\xxy)$ is equivalent to the output of a trained wide network on $X$, where the network is trained on data $(X',Y')$ with learning rate $\lambda$ for $\hat\tau$ steps from the initialization. 


\begin{proof}[Proof Sketch]
Lemma~\ref{lemma:test-predict} is a key lemma used in our analysis, hence we provide a high-level sketch of its proof. The main idea is that, for over-parametrized neural nets, we could approximate the network output function by its first-order Taylor expansion with the corresponding NTKs and NNGPs~\citep{lee2019wide}, provided the network parameters do not have a large displacement during training. Under this case, we can further prove the global convergence of both MTL and ANIL by leveraging tools from \citet{meta-ntk}. The last step is then to analytically compute the corresponding kernels, as shown in Lemma~\ref{lemma:test-predict}.
\end{proof}

With Lemma~\ref{lemma:test-predict}, we proceed to derive the main result in this section. Namely, the predictions given by MTL and ANIL over any test task are close. Intuitively, from \eqref{eq:lemma:test-pred:F_mtl} and \eqref{eq:lemma:test-pred:F_anil}, we can see the test predictions of MTL and ANIL admit a similar form, even though they use different kernels. Inspired by this observation, a natural idea is to bound the difference between the MTL and ANIL kernels by analyzing their spectra, which leads to the following theorem:

\begin{theorem}\label{thm:closeness}
Consider an arbitrary test task, $\task=(\xyxy) \in \bR^{n\times d}\times  \bR^{n\times k}\times \bR^{n'\times d}\times \bR^{n'\times k}$. For any $\epsilon>0$, there exists a constant $h^* = \cO(\epsilon^{-2})$ s.t.\ if the network width $h$ is greater than $h^*$, for ReLU networks with He's initialization, the average difference between the predictions of ANIL and MTL on the query samples $X$ is bounded by 
\begin{align}\label{eq:pred-diff}
    &\quad  \|F_\anil(\xxy) - F_\mtl(\xxy)\|_2 \nonumber \\
    & \leq \cO\left(\lambda \tau  + \frac{1}{L}\right) + \epsilon.
\end{align}
\end{theorem}

\textbf{Remark}~~The bound \eqref{eq:pred-diff} is dominated by $\cO(\lambda \tau + \frac{1}{L})$. Notice that $\lambda$ and $\tau$ are the inner-loop learning rate and adaptation steps of ANIL. In practical implementations, $\lambda \tau \in [0.01,0.5]$, which is small. In the state-of-the-art meta-learning models, the network depth $L \geq 12$, hence $1/L$ is also small. Since the bound holds for \emph{any} test data, it implies that the average discrepancy between the learned predictors of MTL and ANIL is small. Notice that we only study the effect of hyperparameters of models and algorithms (e.g., $L,\lambda, \tau$), and consider dataset-specific parameters (e.g., $N,n,k,d$) as constants.

\begin{proof}[Proof Sketch]
The first step is to apply the analytic forms of $F_\anil$ and $F_\mtl$ in Lemma \ref{lemma:test-predict} to compute their difference. We then prove that the norm of the difference is bounded as
\begin{align*}
& \qquad \|F_\anil(\xxy) - F_\mtl(\xxy) \|_{2} \\
&\leq\cO\left(L \|\NTK(\X,\X)^{-1} - \metaNTK_\mtl^{-1}(\X,\X)\|_{op}\right) + \cO(\lambda \tau+\frac{1}{\sqrt h})
\end{align*}
Then, by leveraging theoretical tools from \citet{xiao2020dis}, we obtain an in-depth structure of the spectrum of the MTL kernel $\metaNTK_\mtl$ for deep ReLU nets, in order to prove that $$ \|\NTK(\X,\X)^{-1} - \metaNTK_\mtl^{-1}(\X,\X)\|_{\op} \leq \cO(\frac {1}{L^2}),$$ with a fine-grained analysis. Finally, defining $h^* = \cO(\eps^{-2})$, we obtain the bound \eqref{eq:pred-diff} for networks with $h > h^*$.
\end{proof}

Theorem~\ref{thm:closeness} could also be extended to ResNets, which have been widely adopted in modern meta-learning applications:
\begin{corollary}\label{corollary:resnets}
\vspace{-1.0em}
For (i) Residual ReLU networks \cite{resnet} and (ii) Residual ReLU networks with Layer Normalization \cite{layernorm}, Theorem \ref{thm:closeness} holds true.
\vspace{-0.5em}
\end{corollary}
\begin{proof}[Proof Sketch]
By leveraging tools from \citet{xiao2020dis}, we show that the residual connection only puts an extra factor $e^L$ on the MTL kernel $\metaNTK_\mtl$. However, plugging it in to the expression for $F_\mtl$ derived in Lemma \ref{lemma:test-predict}, one can find that the extra factors cancel out, since
\begin{align*}
&\qquad e^{L}\metaNTK_\mtl'((X,X',\hat\tau),\X) \cdot (e^{L} \metaNTK_\mtl(\X,\X))^{-1}\\ &=\metaNTK_\mtl'((X,X',\hat\tau),\X)  \metaNTK_\mtl(\X,\X)^{-1}.    
\end{align*}
Similar observation also holds for $\metaNTK_\anil$ and $F_\anil$. Thus, Theorem \ref{thm:closeness} applies to residual ReLU networks as well.

For residual ReLU nets with LayerNorm, $\metaNTK_\mtl$ and $\metaNTK_\anil$ have identical kernel spectra and structures as the regular ReLU nets, up to a difference of a negligible order. Hence, Theorem \ref{thm:closeness} also applies to this class of networks.
\end{proof}

See Appendix \ref{supp:proof} for the full proof of Lemma \ref{lemma:anil-mtl-kernels}, Theorem \ref{thm:closeness} and Corollary \ref{corollary:resnets}.

\section{Experiments}
\label{sec:exp}

In this section, we first provide an empirical validation of Theorem \ref{thm:closeness} on synthetic data. Then, we perform a large-scale empirical study of MTL with unseen task adaptation on few-shot image classification benchmarks to compare with state-of-the-art meta-learning algorithms. The code is released at \url{https://github.com/AI-secure/multi-task-learning}
\vspace{-0.5em}
\subsection{Closeness between MTL and GBML predictions}\label{sec:exp:theory-validate}
\begin{figure}[t!]
    \centering
    \hbox{\hspace{+1.25em}\includegraphics[width=0.85\linewidth]{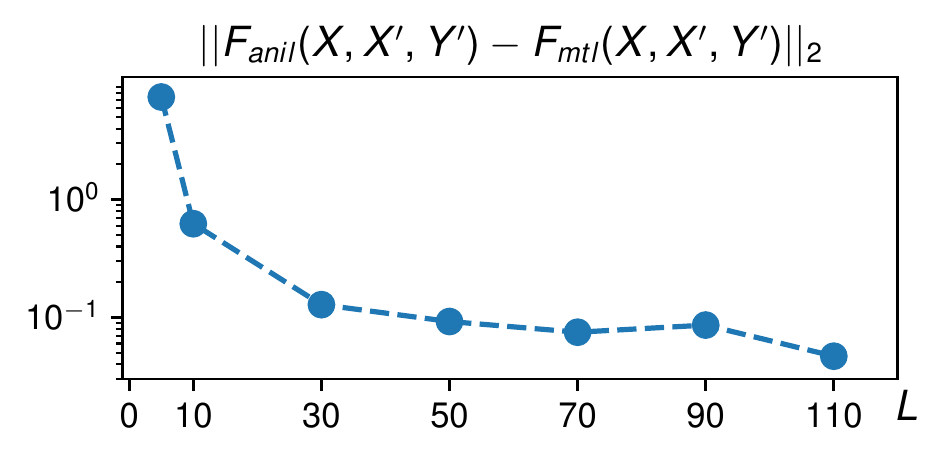}}
    \includegraphics[width=0.84\linewidth]{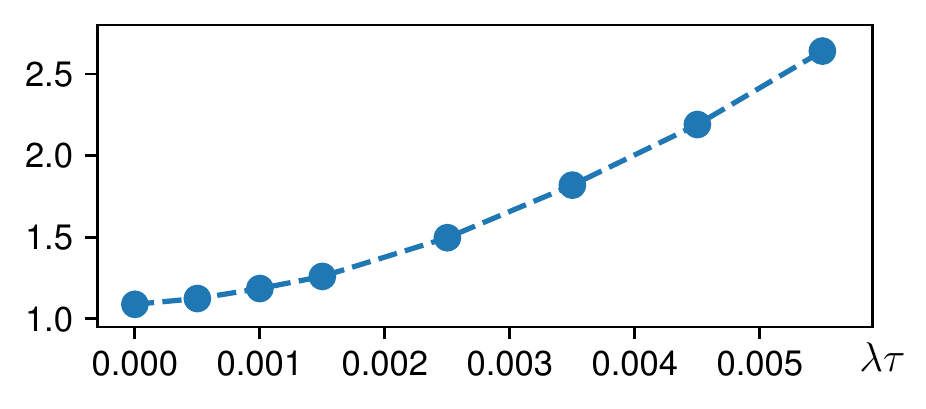}
    \vspace*{-0.2em}
    \caption{Empirical validation of Theorem \ref{thm:closeness} on synthetic data. We vary $L$ in the first figure with fixed $\lambda \tau = 0$, and vary $\lambda \tau$ in the two figures with fixed $L=10$, to observe the corresponding trends in the prediction difference $\|F_\anil(\xxy) - F_\mtl(\xxy)\|_2$.}\label{fig:thm-validate}
\vspace{-1em}
\end{figure}

\textbf{Problem Setting}~~We consider a few-shot regression problem to verify the theoretical claims in Theorem~\ref{thm:closeness}, and adopt the notation defined in Sec.~\ref{sec:prelim}. For each training task $i$ with data $(X_i,Y_i)$, it has two task-specific parameters $\mu_i\in \bR^d$ and $\nu_i\in \bR+$. The data points in $X_i$ are sampled i.i.d. from $\mathcal{N}(\mu_i,\nu_i^2 I)$, and the label of each point $x$ is generated by a quadratic function $y = \nu_i(x-\mu_i)^2$. Similarly, any test task $\task=(\xyxy)$ also has its task-specific parameters $(\mu_\test,\nu_\test)$, and the points from its query and support set $X,X'$ are drawn i.i.d. from $\mathcal{N}(\mu_\test,\nu_\test^2 I)$, with the label of each point $x$ following $y = \nu_\test(x-\mu_\test)^2$.

\textbf{Dataset Synthesis}~~We fix the input dimension $d=10$ and generate $N=20$ training tasks each with $n=10$ data points. In each test task, there are $5$ support and $10$ query data points. For each training or test task, its task-specific parameters $(\mu,\nu)$ are generated by $\mu\sim \mathcal{N}(0,I)$ and $\nu\sim \mathrm{Unif}(1.3,1.6)$. 

\textbf{Implementation Details}~~We implement the functions $F_\mtl$ and $F_\anil$ in \eqref{eq:lemma:test-pred:F_mtl} and \eqref{eq:lemma:test-pred:F_anil} by using the empirical kernel functions of NTK and NNGP provided by Neural Tangents \cite{neuraltangents2020}. As suggested by \citet{neuraltangents2020}, we construct neural nets with width as 512 to compute kernels. Following Sec.~\ref{sec:mtl=meta:functional}, the networks use the ReLU activation and He's initialization \cite{resnet}.

\textbf{Results}~~We generate $20$ test tasks over 5 runs for the empirical evaluation, and we vary the values of $\lambda \tau$ and $L$ appearing in the bound \eqref{eq:pred-diff} of Theorem \ref{thm:closeness}. Figure \ref{fig:thm-validate} shows that as $\lambda \tau$ decreases or $L$ increases, the norm of the prediction difference $\|F_\mtl(\xxy) - F_\anil(\xxy)\|_2$ decreases correspondingly, which is in agreement with \eqref{eq:pred-diff}. More experimental details can be found in Appendix \ref{supp:exp}.

Note that Theorem \ref{thm:closeness} is built on fully connected nets, thus it is not directly applicable to modern convolutional neural nets (ConvNets) with residual connections, max pooling, BatchNorm, and Dropout, which are commonly used in meta-learning practice. Hence, we perform another empirical study on modern ConvNets in Sec. \ref{sec:exp:few-shot}.
\vspace{-.7em}
\subsection{Few-Shot Learning Benchmarks}\label{sec:exp:few-shot}
We conduct experiments on a set of widely used benchmarks for few-shot image classification: mini-ImageNet, tiered-ImageNet, CIFAR-FS and FC100. The first two are derivatives of ImageNet \cite{imagenet}, while the last two are derivatives of CIFAR-100 \cite{cifar}.
\begin{table*}[ht!]

    \caption{
    \textbf{Comparison on four few-shot image classification benchmarks.} Average few-shot test classification accuracy (\%) with 95\% confidence intervals. 32-32-32-32 denotes a 4-layer convolutional neural net with 32 filters in each layer. In each column, \textbf{bold} values are the highest accuracy, or the accuracy no less than $1\%$ compared with the highest one. \\}
    \vspace{-10pt}
    \label{tab:benchmark}
    \vspace{-5pt}
    \begin{center}
    \resizebox{0.95\linewidth}{!}{
    
    \begin{small}
    \begin{tabular}{@{}llc@{}cc@{}c@{}cc@{}}
    \\
    \hline
    \toprule
    & & \phantom{a} & \multicolumn{2}{c}{\textbf{mini-ImageNet 5-way}} & \phantom{ab} & \multicolumn{2}{c}{\textbf{tiered-ImageNet 5-way}} \\
    \cmidrule{4-5} \cmidrule{7-8}
    \textbf{Model} & \textbf{Backbone} && \textbf{1-shot} & \textbf{5-shot} && \textbf{1-shot} & \textbf{5-shot}  \\
    \midrule
    MAML \cite{maml} & 32-32-32-32 &&  48.70 $\pm$ 1.84 & 63.11 $\pm$ 0.92 && 51.67 $\pm$ 1.81 & 70.30 $\pm$ 1.75 \\
    ANIL \cite{raghu2019rapid} & 32-32-32-32 && 48.0 $\pm$ 0.7 & 62.2 $\pm$ 0.5 && - & - \\
    
    
    
    
    
    
    
    
    
    
    R2D2~\cite{r2d2} & 96-192-384-512 && 51.2 $\pm$ 0.6 & 68.8 $\pm$ 0.1 && - & -\\


    
    TADAM \cite{NEURIPS2018_66808e32} & ResNet-12 && 58.50 $\pm$ 0.30 & 76.70 $\pm$ 0.30 && - & - \\
    
    
    
    
    MetaOptNet~\cite{metaOptNet} & ResNet-12 && \textbf{62.64 $\pm$ 0.61} & \textbf{78.63 $\pm$ 0.46 }&& 65.99 $\pm$ 0.72 & 81.56 $\pm$ 0.53 \\
    
    

    \midrule
    MTL-ours & ResNet-12 &&  59.84 $\pm$ 0.22  &  \textbf{77.72 $\pm$ 0.09}  &&  \textbf{67.11 $\pm$ 0.12}  &  \textbf{83.69 $\pm$ 0.02} \\
    \bottomrule
    \hline
    \end{tabular}
    \end{small}
    }
    
    \vspace{+3pt}
    
    \resizebox{0.85\linewidth}{!}{
    \begin{tabular}{@{}llc@{}cc@{}c@{}cc@{}}
    & & \phantom{a} & \multicolumn{2}{c}{\textbf{CIFAR-FS 5-way}} & \phantom{ab} & \multicolumn{2}{c}{\textbf{FC100 5-way}} \\
    \cmidrule{4-5} \cmidrule{7-8}
    \textbf{Model} & \textbf{Backbone} && \textbf{1-shot} & \textbf{5-shot} && \textbf{1-shot} & \textbf{5-shot}  \\
    
    \midrule
    
    MAML \cite{maml} & 32-32-32-32 &&  58.9 $\pm$ 1.9  & 71.5 $\pm$ 1.0  && - & - \\
    R2D2 \cite{r2d2} & 96-192-384-512 && 65.3 $\pm$ 0.2 & 79.4 $\pm$ 0.1 && - & -\\
    TADAM \cite{NEURIPS2018_66808e32} & ResNet-12 && - & - && 40.1 $\pm$ 0.4 & 56.1 $\pm$ 0.4\\
    
    
    
    ProtoNet \cite{snell2017prototypical}& ResNet-12 && \textbf{72.2 $\pm$ 0.7} & \textbf{83.5 $\pm$ 0.5} && 37.5 $\pm$ 0.6 & 52.5 $\pm$ 0.6 \\
    MetaOptNet~\cite{metaOptNet} & ResNet-12 && \textbf{72.6 $\pm$ 0.7} & \textbf{84.3 $\pm$ 0.5} && 41.1 $\pm$ 0.6 & 55.5 $\pm$ 0.6 \\

    \midrule
    MTL-ours & ResNet-12 && 69.5 $\pm$ 0.3 &  \textbf{84.1 $\pm$ 0.1}  && \textbf{42.4 $\pm$ 0.2} &  \textbf{57.7 $\pm$ 0.3}\\
        
    \bottomrule
    \hline
    \end{tabular}
    }
    \end{center}
\end{table*}
\textbf{Benchmarks.}
\vspace{-.7em}
\begin{itemize}[leftmargin=*,align=left]
    \item mini-ImageNet \cite{matching-net}: It contains 60,000 colored images of 84x84 pixels, with 100 classes (each with 600 images) split into 64 training classes, 16 validation classes and 20 test classes.
    \item tiered-ImageNet \cite{ren2018metalearning}: It contains 779,165 colored images of 84x84 pixels, with 608 classes split into 351 training, 97 validation and 160 test classes.
    \item CIFAR-FS \cite{r2d2}: It contains 60,000 colored images of 32x32 pixels, with 100 classes (each with 600 images) split into 64 training classes, 16 validation classes and 20 test classes.
    \item FC100 \cite{NEURIPS2018_66808e32}: It contains 60,000 colored images of 32x32 pixels, with 100 classes split into 60 training classes, 20 validation classes and 20 test classes.
\end{itemize}

\textbf{Network Architecture}~~Following previous meta-learning works \cite{metaOptNet,NEURIPS2018_66808e32,tian2020rethink}, we use ResNet-12 as the backbone, which is a residual neural network with 12 layers \cite{resnet}. 

\textbf{Data Augmentation}~~In training, we adopt the data augmentation used in \citet{metaOptNet} that consists of random cropping, color jittering, and random horizontal flip.

\textbf{Optimization Setup}~~We use RAdam \cite{liu2019radam}, a variant of Adam \cite{adam}, as the optimizer for MTL. We adopt a public PyTorch implementation\footnote{ \url{https://github.com/jettify/pytorch-optimizer}}, and use the default hyper-parameters. Besides, we adopt the ReduceOnPlateau learning rate scheduler\footnote{ \url{https://pytorch.org/docs/stable/optim.html\#torch.optim.lr_scheduler.ReduceLROnPlateau}} with the early stopping regularization\footnote{We stop the training if the validation accuracy does not increase for several epochs.}.

\textbf{Model Selection.} At the end of each training epoch, we evaluate the validation accuracy of the trained MTL model and save a model checkpoint. After training, we select the model checkpoint with the highest validation accuracy, and evaluate it on the test set to obtain the test accuracy.

\textbf{Feature Normalization}~~Following a previous work on few-shot image classification \cite{tian2020rethink}, we normalize features (i.e., last hidden layer outputs) in the meta-test and meta-validations stages. Besides, we also find the feature normalization is effective to the training of MTL on most benchmarks\footnote{It is effective on mini-ImageNet, tiered-ImageNet, and CIFAR-FS, while being ineffective on FC100.}, which might be due to the effectiveness of feature normalization for representation learning \citep{wang2020understand}.

\textbf{Fine-Tuning for Task Adaptation}~~In the meta-validation and meta-testing stages, following Sec. \ref{sec:prelim:fine-tune}, we fine-tune a linear classifier on the outputs of the last hidden layer with the cross-entropy loss. We use the logistic regression classifier with $\ell_2$ regularization from \texttt{scikit-learn} for the fine-tuning \citep{sklearn}. An ablation study on the $\ell_2$ regularization is provided in Appendix \ref{supp:exp:few-shot}.

\textbf{Implementation Details}~~Our implementation is built on the \texttt{learn2learn}\footnote{ \url{http://learn2learn.net/}} package \cite{learn2learn2019}, which provides data loaders and other utilities for meta-learning in PyTorch \cite{pytorch}. We implement MTL on a multi-head version of ResNet-12. Notice that the number of distinct training tasks is combinatorial for 5-way classification on the considered benchmarks, e.g., $\binom{64}{5}=7.6\times 10^{6}$ for mini-ImageNet and $\binom{351}{5}=4.3\times 10^{9}$ for tiered-ImageNet. Hence, due to memory constraints, we cannot construct separate heads for all tasks. Thus, we devise a memory-efficient implementation of the multi-head structure. For instance, on tiered-ImageNet with  351 training classes, we construct a 351-way linear classifier on top of the last hidden layer. Then, for each training task of 5 classes, we select the 5 corresponding row vectors in the weight matrix of the 351-way linear classifier, and merge them to obtain a 5-way linear classifier for this training task.

\textbf{Empirical Results}~~During meta-testing, we evaluate MTL over 3 runs with different random seeds, and report the mean accuracy with the 95\% confidence interval in Table \ref{tab:benchmark}. The accuracy for each run is computed as the mean accuracy over 2000 tasks randomly sampled from the test set. The model selection is made on the validation set.

\textbf{Performance Comparison}~~In Table \ref{tab:benchmark}, we compare MTL with a set of popular meta-learning algorithms on the four benchmarks, in the common setting of 5-way few-shot classification. Notice that MetaOptNet is a state-of-the-art GBML algorithm, and MTL is competitive against it on these benchmarks: across the 8 columns/settings of Table \ref{tab:benchmark}, MTL is worse than MetaOptNet in 2 columns, comparable with MetaOptNet in 2 columns, and outperforms MetaOptNet in 4 columns. Therefore, we can conclude that MTL is competitive with the state-of-the-art of GBML algorithms on few-shot image classification benchmarks.

\textbf{Training Efficiency}~~GBML algorithms are known to be computationally expensive due to the costly second-order bi-level optimization they generally take \cite{hospedales2020metalearning}. In contrast, MTL uses first-order optimization, and as a result, the training of MTL is significantly more efficient. To illustrate this more concretely, we compare the training cost of MTL against MetaOptNet on a AWS server with 4x Nvidia V100 GPU cards\footnote{The \texttt{p3.8xlarge} instance in AWS EC2: {\url{https://aws.amazon.com/ec2/instance-types/p3/}}}. For MetaOptNet \cite{metaOptNet}, we directly run the official PyTorch code\footnote{{\url{https://github.com/kjunelee/MetaOptNet/}}} with the optimal hyper-parameters\footnote{The optimal hyperparameters of MetaOptNet for mini-ImageNet and tiered-Imagenet involve a large batch size that requires 4 GPUs.} provided by the authors. Since the implementations of MetaOptNet and MTL are both written in PyTorch with the same network structure and similar data loaders (both adopting TorchVision dataset wrappers), we believe the efficiency comparison is fair. Note that the two ImageNet derivatives (i.e., mini-ImageNet of 7.2 GB and tiered-ImageNet of 29 GB) are much bigger than that of the two CIFAR-100 derivatives (i.e., CIFAR-FS of 336 MB and FC100 of 336 MB). It is more practically meaningful to reduce the training cost on big datasets like the ImageNet derivatives, thus we only perform the efficiency comparison on mini-ImageNet and tiered-ImageNet.
 \begin{table}[t!]
     \caption{Efficiency Comparison on mini-ImageNet for 5-way 5-shot classification.}
    \label{tab:gpu-hour:all}
    \vspace{+10pt}
    \centering
    \begin{tabular}{c c c }
        \toprule
         & Test Accuracy & GPU Hours \\
        \midrule
        MetaOptNet & 78.63\%  & 85.6 hrs\\
        \midrule
        MTL & 77.72\% & 3.7 hrs\\
        \bottomrule
    \end{tabular}
\end{table}

In Table \ref{tab:gpu-hour:all}, we present the GPU hours for the training of MetaOptNet and MTL with optimal hyper-parameters on mini-ImageNet, showing that the training of MTL is 23x times faster compared with MetaOptNet. 

Figure \ref{fig:speedup-tiered-imagenet} shows the \textit{efficiency-accuracy tradeoff} of MTL vs. MetaOptNet on tiered-Imagenet. The training of MetaOptNet takes $63$ GPU hours, while MTL has various training costs depending on the batch size and the number of epochs. From Figure \ref{fig:speedup-tiered-imagenet} we can see that, even though MTL is only $3.6$x faster when achieving the optimal test accuracy, we can train MTL with a smaller number of epochs or batch size, which reduces the training time at the cost of a small performance drop ($\le 2.2\%$). As shown in Figure \ref{fig:speedup-tiered-imagenet}, while the training of MTL is $11$x faster compared with MetaOptNet, its test accuracy ($81.55\%$) can still match MetaOptNet ($81.56\%$).
\begin{figure}[t!]
    \centering
    \includegraphics[width=0.85\linewidth]{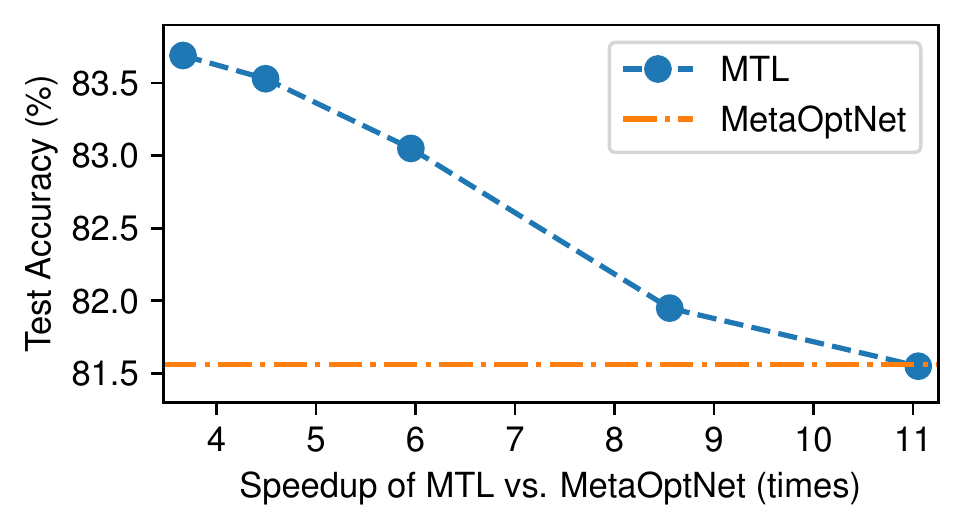}
    \vspace{-0.7em}
    \caption{Efficiency comparison on tiered-ImageNet for 5-way 5-shot classification. The x-axis is the speedup of MTL compared with MetaOptNet, and the y-axis is the mean test accuracy. Notice that we only tune the batch size and number of training epochs for MTL in this comparison.}
    \label{fig:speedup-tiered-imagenet}
    \vspace{-1em}
\end{figure}

\textbf{Remarks on the Empirical Results}~~Traditionally, the training tasks and test tasks for MTL are the same. Our empirical results on few-shot learning reveal that, even as the test tasks are distinct from the training tasks, MTL can still be quite powerful. Recent theoretical studies on MTL show that the joint training of MTL over diverse tasks can learn representations useful for unseen tasks \citep{tripuraneni2020theory,du2021fewshot}, and our few-shot learning experiment supports these theories with positive empirical results. On the other hand, the MTL model we implemented is quite simple, which can be viewed as the original MTL proposal \citep{caruana1997multitask} with a new memory-efficient trick. It is likely that more advanced variants of MTL could achieve even better performance on few-shot learning.
\vspace{-0.5em}

\section{Conclusion}
In this paper, we take an important step towards bridging the gap between MTL and meta-learning, both theoretically and empirically. Theoretically, we show that MTL and gradient-based meta-learning (GBML) share the same optimization formulation. We then further prove that, with sufficiently wide neural networks, the learned predictors from both algorithms give similar predictions on unseen tasks, which implies that it is possible to achieve fast adaptation and efficient training simultaneously. Inspired by our theoretical findings, empirically, we develop a variant of MTL that allows adaptation to unseen tasks, and show that it is competitive against the state-of-the-art GBML algorithms over a set of few-shot learning benchmarks while being significantly more efficient. We believe our work contributes to opening a new path towards models that simultaneously allow efficient training and fast adaptation. 

\vspace{-1em}
\section*{Acknowledgements}
Haoxiang Wang would like to thank Sébastien Arnold and Ruoyu Sun for helpful discussions. This work is partially supported by NSF grant No.1910100, NSF CNS 20-46726 CAR, and Amazon Research Award.

\newpage
\bibliography{main}
\bibliographystyle{icml2021}


\onecolumn
\newpage
\appendix
\icmltitle{Appendix}
\section*{Overview of the Appendix}
The appendix mainly consists of three parts. In Section~\ref{supp:background} we provide more detailed introduction to the set up of meta-learning as well as neural tangent kernels that are missing from the main text due to page limit. In Section~\ref{supp:proof} we provide all the missing proofs of the lemmas and theorems presented in the main paper. In Section~\ref{supp:exp} we discuss in depth about the experiments in the paper. For the convenience of readers, we also provide a copy of the reference at the end of this appendix. 

\section{More on Meta-Learning and Neural Net Setup}\label{supp:background}
In this section, we will provide more information on
\begin{itemize}
    \item Appendix \ref{supp:background:query-support-split}: Query-support split of meta-learning.
    \item Appendix \ref{supp:background:unified-framework}: Unified framework for gradient-based meta-learning that optimizes all layers in the inner loop.
    \item Appendix \ref{supp:background:ntk-parameterization}: NTK parameterization.
\end{itemize}

\subsection{Query-Support Split}\label{supp:background:query-support-split}
Sec. \ref{sec:prelim:mtl} introduces meta-training in the setting without query-support split. In this section, we adopt the notation of Sec. \ref{sec:prelim:mtl}, and describe meta-training in the setting with query-support split below.

The $n$ labelled samples in each training task is divided into two sets, $n_q$ query samples and $n_s$ support samples, i.e., for $i\in[N]$, the $i$-th task consists of 
\begin{align*}
&\text{$n_q$ Query Samples \& Labels:} &X_i^q \in \bR^{n_q \times d}, ~Y_i^q \in \bR^{n_q k}\\
&\text{$n_s$ Support Samples \& Labels:} &X_i^s \in \bR^{n_s \times d},~ Y_i^s \in \bR^{n_q k}
\end{align*}
The optimization objective of ANIL on the training data $\{X_i^q, Y_i^q, X_i^s, Y_i^s\}_{i=1}^N$ is
\begin{align}
    &\min_{\theta} \loss_\anil(\theta) \coloneqq \sum_{i\in[N]} \ell(\phitop_{\theta^\body}(X_i^q) ~ w_i', ~Y_i^q) \\
    &\text{s.t. }w_i'=~\innerloop(w, \phitop_{\theta^\body}(X_i^s), Y_i^s, \tau,\lambda)
\end{align}
It is clear that the $\innerloop$ operation is performed on \textit{support} data $(X_i^s,Y_i^s)$, while the loss evaluation is on the \textit{query} data $(X_i^q,Y_i^q)$.

\subsection{Unified Framework for Gradient-Based Meta-Learning that Optimizes All Layers in the Inner Loop}\label{supp:background:unified-framework}

For GBML algorithms that optimize all layers in the inner loop, their objectives can be summarized into the following unified framework. In contrast to \eqref{eq:meta=mtl:meta-obj}, we have
\begin{align}\label{eq:meta=mtl:meta-obj:all-layers}
    \min_{\theta} \overbrace{\Big[\min_{\{\theta_i\}_{i=1}^N} \sum_{i\in[N]} \ell \Big(f_{\theta_i}(X_i), Y_{i}\Big) + R(\theta_i)\Big]}^{\loss_\gbml(\theta)}~.
\end{align}
Note that similar to~\eqref{eq:meta=mtl:meta-obj}, the parameters $\{\theta_i\}_{i=1}^N$ in~\eqref{eq:meta=mtl:meta-obj:all-layers} are transient, in the sense that GBML algorithms do not explicitly save them during training. In contrast, $\theta$ contains all parameters to optimize in \eqref{eq:meta=mtl:meta-obj:all-layers}, and $\theta$ is optimized over $\ell_\gbml(\theta)$, which is obtained by plugging in the minimizer of $\{\theta_i\}_{i=1}^N$ on the regularized loss. In other words, \eqref{eq:meta=mtl:meta-obj:all-layers} is a bi-level optimization problem, with outer-loop optimization on network parameters $\theta$ and inner-loop optimization on the transient parameters $\{\theta_i\}_{i=1}^N$.

\subsection{NTK Parameterization}\label{supp:background:ntk-parameterization}
NTK parameterization is a neural net parameterization that can be used to provide theoretical analyses of neural net optimization and convergence \cite{lee2019wide,xiao2020dis}. The training dynamics and predictions of NTK-parameterized neural nets are the same as those of standard neural nets \cite{lee2019wide}, up to a width-dependent factor in the learning rate. In what follows, we take a single-head neural net as an example to describe the NTK parameterization. Notice that multi-head networks share the same parameterization with single-head networks, and the only difference is that $N$-head networks have $N$ copies of the output heads (parameterized in the same way as the output heads of single-head networks).

In this paper, we consider a fully-connected feed-forward network with $L$
layers. Each hidden layer has width $l_{i}$, for $i = 1, ..., L - 1$. The readout layer (i.e., output layer) has width $l_{L} = k$. At each layer $i$, for arbitrary input $x\in \mathbb R^{d}$, we denote the pre-activation and post-activation functions by $h^i(x), z^i(x)\in\mathbb R^{l_i}$. The relations between layers in this network are
\begin{align}
\label{eq:recurrence}
\begin{cases}
    h^{i+1}&=z^i W^{i+1} + b^{i+1}
    \\
    z^{i+1}&=\activation \left(h^{i+1}\right) 
    \end{cases}
    \,\, \textrm{and} 
    \,\,
    \begin{cases}
  W^i_{\mu,\nu}& =  \omega_{\mu \nu}^i \sim \mathcal{N} (0, \frac {\sigma_\omega} {\sqrt{l_i}} )
    \\
    b_\nu^i &= \beta_\nu^i \sim \mathcal{N} (0,\sigma_b  )
\end{cases}
,
\end{align}
where $W^{i+1}\in \mathbb R^{l_i\times l_{i+1}}$ and $b^{i+1}\in\mathbb R^{l_{i+1}}$ are the weight and bias of the layer, $\omega_{\mu \nu}^l$ and $ b_\nu^l $ are trainable variables drawn i.i.d. from zero-mean Gaussian distributions at initialization (i.e.,
$\frac{\sws}{l_i}$ and $\sbs$ are variances for weight and bias, and $\activation$ is a point-wise activation function.
\section{Proof}\label{supp:proof}
We present all the missing proofs from the main paper, summarized as follows:
\begin{itemize}
    \item Appendix \ref{supp:proof:global-convergence}: Proves the \textbf{global convergence} of MTL and ANIL, and demonstrates that neural net output and meta-output functions are linearized under over-parameterization.
    \item Appendix \ref{supp:proof:dynamics}: Studies the \textbf{training dynamics} of MTL and ANIL, and derives analytic expressions for their predictors.
    \item Appendix \ref{supp:proof:predictors&kernels}: Derives the expression of \textbf{kernels} for MTL and ANIL, and proves \textbf{Lemma \ref{lemma:test-predict}}.
    \item Appendix \ref{supp:proof:relu-kernel-structures}: Characterizes the \textbf{structures and spectra} of ANIL and MTL kernels for deep \textbf{ReLU} nets.
    \item Appendix \ref{supp:proof:main-theorem}: Proves our main theorem, i.e., \textbf{Theorem \ref{thm:closeness}.}
    \item Appendix \ref{supp:proof:residual}: Extends Theorem \ref{thm:closeness} to \textbf{residual} ReLU networks.
\end{itemize}

\textbf{Shorthand.} As described in Sec. \ref{sec:prelim:fine-tune}, for both MTL and ANIL, we randomly initialize a test head $\wtest$ for fine-tuning in the test phase. Now, we define the following shorthand for convenience.
\begin{itemize}
    \item $\thetatest = \{\theta^\body, \wtest\}$: a parameter set including first $L-1$ layers' parameters of $\theta$ and the test head $\wtest$. 
    \item $\thetamtltest = \{\thetamtl^\body, \wmtltest\}$: a parameter set including first $L-1$ layers' parameters of $\theta$ and the test head $\wtest$.
\end{itemize}

\subsection{Global Convergence of ANIL and MTL with Over-parameterized Deep Neural Nets}\label{supp:proof:global-convergence}
Throughout the paper, we use the squared loss as the objective function of training neural nets: $\ell (\hat y, y)\defeq\frac{1}{2}\|\hat y - y\|_2^2$. To ease the presentation, we define the following meta-output functions.
\begin{definition}[Meta-Output Functions] On any task $\task = (\xyxy)$, for the given adaptation steps $\tau$, we define the meta-output function as 
     \begin{align}
        \label{eq:meta-output}
        F^\tau_\theta(\xxy) = f_{\thetatest}\left(\sX\right) \in \mathbb R^{n k}
        \end{align}
        where the adapted parameters $\thetatest$ is obtained as follows: 
        use $\theta$ as the initial parameter and update it by $\tau$ steps of gradient descent on support samples and labels $(\sX',\sY')$, with learning rate $\lambda$ and loss function $\ell$. Mathematically, $\forall j=0,...,\tau-1$, we have
        \begin{align}\label{eq:meta-adaption-descrete}
            &\theta=\theta_0, ~~\thetatest=\theta_\tau, \text{ and }\theta_{j+1} = \theta_{j} - \lambda \nabla_{\theta_{j}} \ell(f_{\theta_j}(\sX'),\sY')
        \end{align}
\end{definition}

\paragraph{Shorthand} To make the notation uncluttered, we define some shorthand for the meta-output function,
    \begin{itemize}
        \item $F^\tau_\theta(\X,\X,\Y) \triangleq \left(F^\tau_\theta(X_i,X_i,Y_i)\right)_{i=1}^N$: the concatenation of meta-outputs on \textit{all} training tasks. 
        \item $F^\tau_t \triangleq F^\tau_{\theta_t}$: shorthand for the meta-output function with parameters $\theta_t$ at training time $t$.
    \end{itemize}

\paragraph{ANIL Loss} 
With the squared loss function, the training objective of ANIL is expressed as
\begin{align} \label{eq:MAML-obj}
\mathcal L_\anil(\theta) &= \sum_{i=1}^N \ell (F^\tau_\theta(X_i,X_i,Y_i), Y_i) =\frac{1}{2}\sum_{i=1}^N\|F^\tau_\theta(X_i,X_i,Y_i)-Y_i\|_2^2
=\frac{1}{2}\|F^\tau_\theta(\X,\X,\Y)-\Y\|_2^2 
\end{align}
\paragraph{MTL Loss} 
With the squared loss function, the objective of MTL is
\begin{align}
    \mathcal L_\mtl (\thetamtl) &= \sum_{i=1}^N \ell (\fmtl_\thetamtl(X_i,i), Y_i)
    =\frac{1}{2}\sum_{i=1}^N\|\fmtl_\thetamtl(X_i,i)-Y_i\|_2^2 
    =\frac{1}{2}\|\fmtl_\thetamtl(\X)-\Y\|_2^2 
\end{align}
where we define the notation $\fmtl_\thetamtl(\X)$ to be $\vecop(\{\fmtl_\thetamtl(X_i, i)\}_{i=1}^N)$.
\paragraph{Tangent Kernels} Now, we define tangent kernels for MTL and ANIL, following \citet{meta-ntk}. Denote $\width$ as the minimum width across hidden layers, i.e., $\width = \min_{l\in[L-1]} h_l$. Then, the tangent kernels of MTL and ANIL are defined as
\begin{align}
        \metaNTK_\mtl  &= \lim_{\width \rightarrow \infty} \nabla_{\thetamtl_0} \fmtl_{\thetamtl_0}(\X) \cdot  \nabla_{\thetamtl_0} \fmtl_{\thetamtl_0}(\X)^\top\\
        \metaNTK_\anil  &= \lim_{\width\rightarrow \infty}\nabla_{\theta_0} F^\tau_{\theta_0}(\X,\X,\Y) \cdot  \nabla_{\theta_0} F^\tau_{\theta_0}(\X,\X,\Y)^\top
\end{align}
Notice that by \citet{meta-ntk}, we know both kernels are deterministic positive-definite matrices, independent of the initializations $\theta_0$ and $\thetamtl_0$.

Next, we present the following theorem that characterizes the global convergence of the above two algorithms on over-parametrized neural networks.  

\begin{theorem}[Global Convergence of ANIL and MTL with Over-parameterized Deep Neural Nets] \label{thm:global-convergence}
Define 
\begin{equation*}
\eta_0 = \min\left\{\frac{2}{\lev(\metaNTK_\mtl)+\Lev(\metaNTK_\anil)},\frac{2}{\lev(\metaNTK_\mtl)+\Lev(\metaNTK_\anil)}\right\}.    
\end{equation*}
For arbitrarily small $\delta > 0 $, there exists constants $R,\lambda_0,\width^* > 0$ such that for networks with width greater than $\width^*$, running gradient descent on $\loss_\mtl$  and $\loss_\anil$ with learning rate $\eta > \eta_0$ and inner-loop learning rate $\lambda < \lambda_0$, the following bounds on training losses hold true with probability at least $1-\delta$ over random initialization,
\begin{align}
    \loss_\anil (\theta_t) &\leq \left(1- \frac{1}{3}\eta_0\cdot \lev(\metaNTK_\anil)\right)^{2t} R \\
    \loss_\mtl (\thetamtl_t) &\leq \left(1- \frac{1}{3}\eta_0\cdot \lev(\metaNTK_\mtl)\right)^{2t} R
\end{align}
where $t\in \mathbb{N}$ is the number of training steps. Furthermore, the displacement of the parameters during the training process can be bounded by
\begin{align}\label{eq:thm:global-convergence:parameter-movement}
    \sup_{t\geq 0} \frac{1}{\sqrt{\width}} \|\theta_t - \theta_0\|_2 = \cO(\width^{-\frac{1}{2}}), ~~\sup_{t\geq 0} \frac{1}{\sqrt{\width}} \|\thetamtl_t - \thetamtl_0\|_2 = \cO(\width^{-\frac{1}{2}})
\end{align}
\end{theorem}
\textbf{Remarks.} Notice the bounds in \eqref{eq:thm:global-convergence:parameter-movement} are derived in the setting of NTK parameterization (see Appendix \ref{supp:background:ntk-parameterization}). When switching to the standard parameterization, as shown by Theorem G.2 of \citet{lee2019wide}, \eqref{eq:thm:global-convergence:parameter-movement} is transformed to 
\begin{align}
    \sup_{t\geq 0} \|\theta_t - \theta_0\|_2 = \cO(\width^{-\frac{1}{2}}), ~~\sup_{t\geq 0} \|\thetamtl_t - \thetamtl_0\|_2 = \cO(\width^{-\frac{1}{2}}),
\end{align}
indicating a closeness between the initial and trained parameters as the network width $\width$ is large.
\begin{proof}
For ANIL, the global convergence can be straightforwardly obtained by following the same steps of Theorem 4 of \citet{meta-ntk}, which proves the global convergence for MAML in the same setting\footnote{Notice the only difference between ANIL and MAML is the layers to optimize in the inner loop, where ANIL optimizes less layers than MAML. Hence, bounds on the inner loop optimization in Theorem 4 of \citet{meta-ntk} cover that of ANIL, and the proof steps of that theorem applies to the case of ANIL.}. 

For MTL, it can be viewed as a variant of MAML with multi-head neural nets and inner-loop learning rate $\tau = 0$, since it only has the outer-loop optimization. Then, the global convergence of MTL can also be straightforwardly obtained by following the proof steps of Theorem 4 from \citet{meta-ntk}.
\end{proof}

\textbf{Linearization at Large Width.} The following corollary provides us a useful toolkit to analyze the training dynamics of both ANIL and MTL in the over-parametrization regime, which is adopted and rephrased from \citet{meta-ntk} and \citet{lee2019wide}.

\begin{corollary}[Linearized (Meta) Output Functions]\label{corollary:linearization}
For arbitrarily small $\delta > 0$, there exists $h^* >0$ s.t. as long as the network width $h$ is greater than $h^*$, during the training of ANIL and MTL, with probability at least $1-\delta$ over random initialization, the network parameters stay in the neighbourhood of the initialization s.t. $\theta_t \in \{\theta: \|\theta - \theta_0\|_2 \leq \cO(1/\width^2) \}$ or $\thetamtl_t \in \{\thetamtl: \|\thetamtl - \thetamtl_0\|_2 \leq \cO(1/\width^2) \}$, where $\theta_0 = \{\theta_0^\body,w_0\}$ and $\thetamtl_0 = \{\theta_0^\body\}\cup \{\wmtl{i}_0\}_{i\in[N]}$ are the initial parameters of networks trained by ANIL and MTL, respectively. Then, for any network trained by ANIL, its output on any $x\in \bR^d$ is effectively linearized, i.e.,
    \begin{align}\label{eq:linear-output-single-head}
        f_\theta(x) = f_{\theta_0}(x) + \nabla_{\theta_0}f_{\theta_0} (x) (\theta - \theta_0) + \cO(\frac{1}{\sqrt{\width}})
    \end{align}
Similarly, for any network trained by MTL, the output of the \textit{multi-head} neural net on $x$ with head index $i\in[N]$ is characterized by
    \begin{align}\label{eq:linear-output-mtl}
        \fmtl_\thetamtl(x,i) = \fmtl_{\thetamtl_0}(x,i) + \nabla_{\thetamtl_0}\fmtl_{\thetamtl_0} (x,i) (\thetamtl - \thetamtl_0) + \cO(\frac{1}{\sqrt{\width}})
    \end{align}
Besides, the meta-output function is also effectively linearized, i.e., for any task $\task = (\xyxy)$,
    \begin{align}\label{eq:linear-output-anil}
        F^\tau_{\theta}(\xxy) &= F^\tau_{\theta_0}(\xxy) + \nabla_{\theta_0} F^\tau_{\theta_0}(\xxy) (\theta - \theta_0) + \cO(\frac{1}{\sqrt{\width}}),
    \end{align}
where $F^\tau_{\theta_0}(\xxy)$ can be expressed as
    \begin{align}\label{eq:F_0}
        F^\tau_{\theta_0}(\xxy) = f_{\theta_0}(X) + \nngp_{w_0}(X,X')\nngp_{w_0}^{-1}(X',X') \left(I - e^{-\lambda \nngp_{w_0}(X',X')\tau}\right)\left[Y' - f_{\theta_0}(X')\right] + \cO(\frac{1}{\sqrt{\width}}),
    \end{align}
and the gradient $\nabla_{\theta_0}F^\tau_{\theta_0}(\xxy)$ as\footnote{The proof of the gradient expression can be straightforwardly obtained by Lemma 6 of \cite{meta-ntk}.}
    \begin{align}\label{eq:F_0-grad}
        \nabla_{\theta_0}F^\tau_{\theta_0}(\xxy) = \nabla_{\theta_0}f_{\theta_0}(X) - \nngp_{w_0}(X,X')\nngp_{w_0}^{-1}(X',X') \left(I - e^{-\lambda \nngp_{w_0}(X',X')\tau}\right) \nabla_{\theta_0}f_{\theta_0}(X') + \cO(\frac{1}{\sqrt{\width}}),
    \end{align}
    with $\nngp_{w_0}$ defined as 
    \begin{align*}
        \nngp_{w_0}(\cdot,\ast) = \nabla_{w} f_{\theta_0}(\cdot) \cdot\nabla_{w} f_{\theta_0}(\ast)^\top
    \end{align*}
\end{corollary}
\textbf{Remarks.} One can replace $\theta_0$ in \eqref{eq:linear-output-anil} with $\{\theta_0^\body, \wtest\}$ or $\{\thetamtl_0^\body, \wmtltest\}$, and similar results apply.
\begin{proof}
    Notice that the proof of Theorem \ref{thm:global-convergence} above is based on Theorem 4 of \citet{meta-ntk}, which also proves that the trained parameters stay in the neighborhood of the initialization with radius of $\cO(\frac{1}{\sqrt{\width}})$. Hence, following the proof steps of Theorem 4 of \citet{meta-ntk}, one can also straightforwardly prove the same result for ANIL and MTL.

    With the global convergence and the neighborhood results above, we can directly invoke Theorem H.1 of \citet{lee2019wide}, and obtain \eqref{eq:linear-output-single-head}, \eqref{eq:linear-output-mtl} and \eqref{eq:linear-output-anil}. Notice, the expressions in \eqref{eq:F_0} and \eqref{eq:F_0-grad} are derived in Sec. 2.3.1 of \citet{lee2019wide}.
\end{proof}

\subsection{Training Dynamics of MTL and ANIL}\label{supp:proof:dynamics}

\begin{definition}[Empirical Tangent Kernels of ANIL and MTL] We define the following empirical tangent kernels of ANIL and MTL, in a similar way to \cite{meta-ntk,lee2019wide}:
    \begin{align}
        \metantk_\anil(\X,\X) &= \nabla_{\theta_0} F^\tau_{\theta_0}(\X,\X,\Y) \cdot \nabla_{\theta_0} F^\tau_{\theta_0}(\X,\X,\Y)^\top \in \bR^{Nn \times Nn}\\
        \metantk_\mtl(\X,\X) &= \nabla_{\theta_0} \fmtl_{\thetamtl_0}(\X) \cdot \nabla_{\theta_0} \fmtl_{\thetamtl_0}(\X)^\top \in \bR^{Nn \times Nn}
    \end{align}
\end{definition}

\textbf{Shorthand.} To simplify expressions, we define the following shorthand. For any kernel function $\metantk$, learning rate $\eta$ and optimization steps $t$, we have
\begin{align}
    T^{\eta,t}_{\metantk}(\cdot) &= \metantk^{-1}(\cdot,\cdot) \left(I - e^{-\eta \metantk (\cdot,\cdot) t}\right)
\end{align}

\begin{lemma}[ANIL and MTL in the Linearization Regime]\label{lemma:lienarization}
With linearized output functions shown in Corollary \ref{corollary:linearization}, the training dynamics of ANIL and MTL under gradient descent on squared losses can be characterized by analytically solvable ODEs, giving rise to the solutions:
\begin{itemize}[leftmargin=*,align=left,noitemsep,nolistsep]
    \item ANIL.
    \begin{itemize}[leftmargin=*,align=left,noitemsep,nolistsep]
        \item Trained parameters at time $t$:
        \begin{align}
            \theta_t &= \theta_0 + \nabla_{\theta_0} F^\tau_{\theta_0}(\X,\X,\Y)^\top \metantk_\anil^{-1}(\X,\X) \left(I - e^{-\eta \metantk_\anil (\X,\X) t}\right) \left[\Y-F^\tau_{\theta_0}(\X,\X,\Y) \right] + \cO(\frac{1}{\sqrt{\width}})
        \end{align}
        \item Prediction on any test task $\task=(\xyxy)$ with adaptation steps $\hat\tau$ (i.e., we take the hidden layers of the trained network $\theta^\body$ and append a randomly initialized head $\wtest$ to fine-tune):
        \begin{align}\label{eq:F_t-test}
            &\qquad F^{\hat\tau}_{\thetatest_t}(\xxy) \\
            & = F^{\hat\tau}_{\thetatest_0}(\xxy) + \nabla_{\theta_0^\body} F^{\hat \tau}_{\thetatest_0} (\xxy) \nabla_{\theta_0^\body} F^{\tau}_{\theta_0} (\X,\X,\Y)^\top  T^{\eta,t}_{\metantk_\anil}(\X) \left[\Y - F^\tau_{\theta_0}(\X,\X,\Y)\right] + \cO(\frac{1}{\sqrt{\width}})\nonumber
        \end{align}
        where $\thetamtltest_t = \{\thetamtl_t^\body, \wmtltest\}$ and $\thetamtltest_0 = \{\thetamtl_0^\body, \wmtltest\}$.
    \end{itemize}
    \item MTL.
    \begin{itemize}[leftmargin=*,align=left,noitemsep,nolistsep]
        \item Trained parameters:
        \begin{align}
            \thetamtl_t 
            &=\thetamtl_0 + \nabla_{\thetamtl_0} \fmtl_{\thetamtl_0}(\X)^\top T^{\eta,t}_{\metantk_\mtl}(\X) \left[\Y - \fmtl_{\thetamtl_0}(\X) \right]+ \cO(\frac{1}{\sqrt{\width}})
        \end{align}
        \item Prediction on test task $\task=(\xyxy)$ with adaptation steps $\hat\tau$ (i.e., we take the hidden layers of the trained network $\thetamtl^\body$ and append a randomly initialized head $\wmtltest$ to fine-tune):
        \begin{align}\label{eq:F_t-test:mtl}
            &\qquad F^{\hat\tau}_{\thetamtltest_t}(\xxy) \nonumber\\
            &= F^{\hat\tau}_{\thetamtltest_0}(\xxy) +\nabla_{\thetamtl_0^\body} F^{\hat \tau}_{\thetamtltest_0} (\xxy) \nabla_{\thetamtl_0^\body} \fmtl_{\thetamtl_0}(\X)^\top T^{\eta,t}_{\metantk_\mtl}(\X) \left[\Y - \fmtl_{\thetamtl_0}(\X)\right] + \cO(\frac{1}{\sqrt{\width}})
        \end{align}
        where $\thetamtltest_t = \{\thetamtl_t^\body, \wmtltest\}$, $\thetamtltest_0 = \{\thetamtl_0^\body, \wmtltest\}$.
    \end{itemize}
\end{itemize}
\end{lemma}
\begin{proof}

Similar to Sec. 2.2 of \citet{lee2019wide}, with linearized functions \eqref{eq:linear-output-mtl} and \eqref{eq:linear-output-anil}, the training dynamics of MTL and ANIL under gradient flow with squared losses are governed by the ODEs,
\begin{itemize}
    \item Training dynamics of ANIL.
    \begin{align}
        \diff{\theta_t}{t} = -\eta \nabla_{\theta_0} F_{\theta_0}(\X,\X,\Y)^\top \left(F_{\theta_t}(\X,\X,\Y) - \Y\right) \\
        \diff{F_{\theta_t}(\X,\X,\Y)}{t} = -\eta \metantk_\anil(\X,\X) \left(F_{\theta_t}(\X,\X,\Y) - \Y\right) 
    \end{align}
    Solving the set of ODEs, we obtain the solution to $\theta_t$ as 
    \begin{align}
        \theta_t = \theta_0 - \nabla_{\theta_0} F_{\theta_0}(\X,\X,\Y)^\top \metantk_\anil(\X,\X)^{-1} \left(I - e^{-\eta \metantk_\anil(\X,\X)t }\right)  \left(F_{\theta_0}(\X,\X,\Y) - \Y\right)
    \end{align}
    up to an error of $\cO(\frac{1}{\sqrt{\width}})$. See Theorem H.1 of \citet{lee2019wide} for the bound on the error across training.
     \item Training dynamics of MTL.
    \begin{align}
        \diff{\thetamtl_t}{t} &= -\eta \nabla_{\thetamtl_0} \fmtl_{\thetamtl_0}(\X)^\top \left(\fmtl_{\theta_t}(\X) - \Y\right)\\
        \diff{\fmtl_{\theta_t}(\X)}{t}&= -\eta \metantk_\mtl(\X,\X) \left(\fmtl_{\theta_0}(\X) - \Y\right)
    \end{align}
    Solving the set of ODEs, we obtain the solution to $\thetamtl_t$ as
    \begin{align}
        \thetamtl_t = \thetamtl_0 - \nabla_{\thetamtl_0} \fmtl_{\thetamtl_0}(\X)^\top \metantk_\mtl(\X,\X)^{-1} \left(I - e^{-\eta \metantk_\mtl(\X,\X)t }\right)  \left(\fmtl_{\theta_t}(\X) - \Y\right)
    \end{align}
    up to an error of $\cO(\frac{1}{\sqrt{\width}})$. See Theorem H.1 of \citet{lee2019wide} for the bound on the error across training.
    
\end{itemize}

Now, with the derived expressions of trained parameters, we can certainly plug them in the linearized functions \eqref{eq:linear-output-mtl} and \eqref{eq:linear-output-anil} to obtain the outputs of trained ANIL and MTL models. Notice that during test, the predictions of ANIL and MTL are obtained from a fine-tuned \textit{test head} that are randomly initialized (see Sec. \ref{sec:prelim:fine-tune} for details). Thus, we need to take care of the test heads when plugging trained parameters into the linearized functions. Specifically, for an arbitrary test task $\task = (\xyxy)$, the test predictions of ANIL and MTL are derived below.
\begin{itemize}
    \item Test predictions of ANIL.
    For notational simplicity, we define 
    $$\nngp_t(\cdot,\ast) = \nabla_{\wtest} f_{\thetatest_t}(\cdot) \nabla_{\wtest}  f_{\thetatest_t}(\ast)^\top$$
    Then, since the fine-tuning is on the test head $\wtest$, following the Sec. 2.3.1. of \citet{lee2019wide}, we know
    \begin{align*}
        F^{\hat\tau}_{\thetatest_t} (\xxy) = f_{\thetatest_t}(X) +  \nngp_t(X,X') T^{\lambda,\hat\tau}_{\nngp_t}(X')\left ( Y' -  f_{\thetatest_t}(X')\right) + \cO(\frac{1}{\sqrt{\width}})\eq
    \end{align*}
    where 
    \begin{align*}
        &\qquad f_{\thetatest_t}(X) \eq\\
        &= f_{\thetatest_0}(X) + \nabla_{\thetatest_0} f_{\thetatest_0}(X) (\thetatest_t - \thetatest_0) + \cO(\frac{1}{\sqrt{\width}})\\
        &=f_{\thetatest_0}(X) +\nabla_{\theta_0^\body} f_{\thetatest_0}(X) (\theta^\body_t - \theta_0^\body) + \nabla_{\wtest} f_{\thetatest_0}(X) (\wtest - \wtest)+ \cO(\frac{1}{\sqrt{\width}}) \\
        &=f_{\thetatest_0}(X) +\nabla_{\theta_0^\body} f_{\thetatest_0}(X) (\theta^\body_t - \theta_0^\body) + \cO(\frac{1}{\sqrt{\width}})\\
        &=f_{\thetatest_0}(X) +\nabla_{\theta_0^\body} f_{\thetatest_0}(X)  \nabla_{\theta_0^\body} F_{\theta_0}(\X,\X,\Y)^\top \metantk_\anil(\X,\X)^{-1} \left(I - e^{-\eta \metantk_\anil(\X,\X)t }\right)  \left(\Y - F_{\theta_0}(\X,\X,\Y)\right)\\
        &\quad + \cO(\frac{1}{\sqrt{\width}}) \\
        &=f_{\thetatest_0}(X) +\nabla_{\theta_0^\body} f_{\thetatest_0}(X)  \nabla_{\theta_0^\body} F_{\theta_0}(\X,\X,\Y)^\top T^{\eta,t}_{\metantk_\anil}(\X) \left(\Y - F_{\theta_0}(\X,\X,\Y)\right)
        + \cO(\frac{1}{\sqrt{\width}}) 
    \end{align*}
    and 
    \begin{align*}
        \nabla_{\wtest} f_{\thetatest_t}(X) 
        = \nabla_{\wtest} f_{\thetatest_0}(X) + \cO(\frac{1}{\sqrt{\width}}) \eq
    \end{align*}
    Pluging in everything, we have
    \begin{align*}
        &\qquad F^{\hat\tau}_{\thetatest_t} (\xxy) \eq\\
        &= f_{\thetatest_0}(X) +\nabla_{\theta_0^\body} f_{\thetatest_0}(X)  \nabla_{\theta_0^\body} F_{\theta_0}(\X,\X,\Y)^\top T^{\eta,t}_{\metantk_\anil}(\X) \left(\Y - F_{\theta_0}(\X,\X,\Y)\right) \\
        &\quad +\nngp_0(X,X') T^{\lambda,\hat\tau}_{\nngp_0}(X')\left ( Y' -  f_{\thetatest_t}(X')\right) + \cO(\frac{1}{\sqrt{\width}})\\
        &= f_{\thetatest_0}(X) + \nngp_0(X,X') T^{\lambda,\hat\tau}_{\nngp_0}(X')(Y'- f_{\thetatest_0}(X')) \\
        &\quad + \left(\nabla_{\theta_0^\body} f_{\thetatest_0}(X)  - \nngp_0(X,X') T^{\lambda,\hat\tau}_{\nngp_0}(X')\nabla_{\theta_0^\body} f_{\thetatest_0}(X')\right)  \nabla_{\theta_0^\body} F_{\theta_0}(\X,\X,\Y)^\top T^{\eta,t}_{\metantk_\anil}(\X) \left(\Y - F_{\theta_0}(\X,\X,\Y)\right) \\
        &\quad +  \cO(\frac{1}{\sqrt{\width}})\\
        &= F^{\hat\tau}_{\thetatest_0}(\xxy) + \nabla_{\theta_0^\body} F^{\hat \tau}_{\thetatest_0} (\xxy) \nabla_{\theta_0^\body} F^{\tau}_{\theta_0} (\X,\X,\Y)^\top  T^{\eta,t}_{\metantk_\anil}(\X) \left[\Y - F^\tau_{\theta_0}(\X,\X,\Y)\right] + \cO(\frac{1}{\sqrt{\width}})
    \end{align*}
    \item Test prediction of MTL. Following the derivation for the test prediction of ANIL above, one can straightforwardly derive that 
    \begin{align*}
    F^{\hat\tau}_{\thetamtltest_t}(\xxy) = F^{\hat\tau}_{\thetamtltest_0}(\xxy) +\nabla_{\thetamtl_0^\body} F^{\hat \tau}_{\thetamtltest_0} (\xxy) \nabla_{\thetamtl_0^\body} \fmtl_{\thetamtl_0}(\X)^\top T^{\eta,t}_{\metantk_\mtl}(\X) \left[\Y - \fmtl_{\thetamtl_0}(\X)\right] + \cO(\frac{1}{\sqrt{\width}})
    \end{align*}
\end{itemize}
\end{proof}


\subsection{Derivation of Kernels and Outputs for ANIL and MTL.}\label{supp:proof:predictors&kernels}
\begin{notation}[NTK and NNGP] We denote
    \begin{itemize}
        \item $\NTK(\cdot,\ast)$: kernel function of Neural Tangent Kernel (NTK).
        \item $\NNGP(\cdot,\ast)$: kernel function of Neural Network Gaussian Process (NNGP).
    \end{itemize}
\end{notation}

\paragraph{Equivalence to Kernels} \citet{lee2019wide} shows that as the network width $h$ approaches infinity, for parameter initialization $\theta_0 = \{\theta_0^\body,w_0\}$, we have the following equivalence relations,
\begin{align}
    \nabla_{\theta_0} f_{\theta_0}(\cdot) \nabla_{\theta_0} f_{\theta_0}(\ast)^\top &= \NTK(\cdot, \ast) \\
    \nabla_{w} f_{\theta_0}(\cdot)  \nabla_{\theta_0} f_{w}(\ast)^\top &= \NNGP(\cdot, \ast)
\end{align}

\begin{lemma}[ANIL and MTL Kernels]\label{lemma:anil-mtl-kernels} As the width of neural nets increases to infinity, i.e., $\width \rightarrow \infty$, we define the following kernels for ANIL and MTL, and they converge to corresponding analytical expressions shown below.
\begin{itemize}
    \item \textbf{ANIL kernels.} 
    \begin{itemize}
        \item $\metaNTK_\anil(\X,\X) = \nabla_{\theta_0} F^\tau_{\theta_0}(\X,\X,\Y) \cdot  \nabla_{\theta_0} F^\tau_{\theta_0}(\X,\X,\Y)^\top$ is a block matrix of $N\times N$ blocks. $\forall i,j\in[N]$, its $(i,j)$-th block is
    \begin{align}
        &\quad [\metaNTK_\mathrm{ANIL}(\X,\X)]_{ij} = e^{-\lambda \NNGP(X_i,X_i)\tau} \NTK(X_i,X_j) e^{-\lambda \NNGP(X_j,X_j)\tau},\nonumber
    \end{align}
        \item $\metaNTK_\anil'((\xx,\hat\tau),\X)=\nabla_{\theta_0^\body} F^{\hat \tau}_{\thetatest_0} (\xxy) \nabla_{\theta_0^\body} F^{\tau}_{\theta_0} (\X,\X,\Y)^\top $ is a block matrix of $1 \times N$ blocks, with the $(1,j)$-th block as 
    \begin{align}
        &\qquad [\metaNTK_\anil'((\xx,\hat\tau),\X)]_{1j} =
         \Big[\NTK(X,X_j) - \NNGP(X,X') \T_{\NNGP}^{\hat\tau}(X') \NTK(X',X_j)\Big]e^{-\lambda \NNGP(X_j,X_j) \tau} \nonumber
    \end{align}
    \end{itemize}
    \item \textbf{MTL Kernels.} 
    \begin{itemize}
        \item $\metaNTK_\mtl(\X,\X)=\nabla_{\thetamtl_0} \fmtl_{\thetamtl_0}(\X) \cdot  \nabla_{\thetamtl_0} \fmtl_{\thetamtl_0}(\X)^\top$ is also a block matrix of $N\times N$ blocks. $\forall i,j\in[N]$, its $(i,j)$-th block is
    \begin{align}
        [\metaNTK_\mtl(\X,\X)]_{ij} = \NTK (X_i,X_j) - \indicator [i\neq j] \NNGP(X_i,X_j), \nonumber
    \end{align}
        \item $\metaNTK_\mtl'((\xx,\hat\tau),\X)=\nabla_{\thetamtl_0^\body} F^{\hat \tau}_{\thetamtltest_0} (\xxy) \nabla_{\thetamtl_0^\body} \fmtl_{\thetamtl_0}(\X)^\top$ is a block matrix of $1 \times N$ blocks, with the $(1,j)$-th block as
    \begin{align}
        [\metaNTK_\mtl'((X,X',\hat\tau),\X)]_{1j} &=
        \NTK(X,X_j) - \NNGP(X,X_j)
         - \NNGP(X,X')\T_{\NNGP}^{\hat\tau}(X') \Big[\NTK(X',X_j) - \NNGP(X',X_j)\Big] \\
        &=\NTK(X,X_j) - \NNGP(X,X') \T_{\NNGP}^{\hat\tau}(X') \NTK(X',X_j) 
        - \NNGP(X,X')e^{-\lambda \NNGP(X',X') \hat \tau}\NNGP(X',X_j)\nonumber
    \end{align}
    \end{itemize}    
\end{itemize}
\end{lemma}
\begin{proof} The proof is presented in the same structure as the lemma statement above.
\begin{itemize}[leftmargin=*,align=left,noitemsep,nolistsep]
    \item \textbf{ANIL Kernels}
    \begin{itemize}[leftmargin=*,align=left,noitemsep,nolistsep]
        \item $\metaNTK_\anil(\X,\X)$. With \eqref{eq:F_0-grad}, we know
        \begin{align}
            \nabla_{\theta_0} F^\tau_{\theta_0}(\X,\X,\Y) &= \left(\nabla_{\theta_0} F^\tau_{\theta_0}(X_i,X_i,Y_i) \right)_{i=1}^N \nonumber\\
            &= \left(\nabla_{\theta_0} f_{\theta_0}(X_i) - \NNGP(X_i,X_i)\NNGP^{-1}(X_i,X_i)\left(I - e^{-\lambda \NNGP(X_i,X_i)\tau}\right) \nabla_{\theta_0} f_{\theta_0}(X_i) \right)_{i=1}^N\nonumber \\
            &=\left(e^{-\lambda \NNGP(X_i,X_i)\tau} \nabla_{\theta_0} f_{\theta_0}(X_i) \right)_{i=1}^N
        \end{align}
        Thus, the $(i,j)$-th block of $\metaNTK_\anil(\X,\X) = \nabla_{\theta_0} F^\tau_{\theta_0}(\X,\X,\Y) \cdot  \nabla_{\theta_0} F^\tau_{\theta_0}(\X,\X,\Y)^\top$ is
        \begin{align}
            \left[\metaNTK_\anil(\X,\X)\right]_{ij} &= \nabla_{\theta_0} F^\tau_{\theta_0}(X_i,X_i,Y_i) \nabla_{\theta_0} F^\tau_{\theta_0}(X_j,X_j,Y_j)^\top \nonumber \\
            &= e^{-\lambda \NNGP(X_i,X_i)\tau} \nabla_{\theta_0} f_{\theta_0}(X_i) f_{\theta_0}(X_j)^\top  e^{-\lambda \NNGP(X_j,X_j)\tau} \nonumber \\
            & = e^{-\lambda \NNGP(X_i,X_i)\tau} \NTK(X_i,X_j)  e^{-\lambda \NNGP(X_j,X_j)\tau}
        \end{align}
        Then, the whole matrix can be expressed as
        \begin{align}\label{eq:anil-kernel:supp}
            \metaNTK_\anil(\X,\X) = \diag{\{e^{-\lambda \NNGP(X_i,X_i)\tau}\}_{i=1}^N}\cdot \NTK(\X,\X) \cdot \diag{\{e^{-\lambda \NNGP(X_j,X_j)\tau} \}_{j=1}^N}
        \end{align}
        where $\diag{\{e^{-\lambda \NNGP(X_i,X_i)\tau}\}_{i=1}^N}$ is a diagonal block matrix with the $i$-th block as $e^{-\lambda \NNGP(X_i,X_i)\tau}$.
        \item $\metaNTK_\anil'((\xx,\hat\tau),\X)$. With \eqref{eq:F_0-grad}, we can derive that 
        \begin{align}
            &\qquad\left[\metaNTK_\anil'((\xx,\hat\tau),\X)\right]_{1j}\nonumber \\
            &=  \nabla_{\theta_0^\body}F^{\hat\tau}_{\thetatest_0}(\xxy) \nabla_{\theta_0^\body}F^{\tau}_{\theta_0}(X_j,X_j,Y_j)^\top\nonumber\\
            &=\left(\nabla_{\theta_0^\body} f_{\thetatest_0}(X) - \NNGP(X,X')\NNGP^{-1}(X',X')\left(I - e^{-\lambda \NNGP(X',X')\hat\tau}\right) \nabla_{\theta_0^\body} f_{\theta_0^\body}(X') \right) \nonumber\\
            &\quad\cdot \left(\nabla_{\theta_0^\body} f_{\theta_0}(X_j) - \NNGP(X_j,X_j)\NNGP^{-1}(X_j,X_j)\left(I - e^{-\lambda \NNGP(X_j,X_j)\tau}\right) \nabla_{\theta_0^\body} f_{\theta_0}(X_j) \right)^\top\nonumber\\
            &= \left(\nabla_{\theta_0^\body} f_{\thetatest_0}(X) - \NNGP(X,X')\NNGP^{-1}(X',X')\left(I - e^{-\lambda \NNGP(X',X')\hat\tau}\right) \nabla_{\theta_0^\body} f_{\theta_0}(X') \right) \cdot \nabla_{\theta_0^\body} f_{\theta_0}(X_j)^\top e^{-\lambda \NNGP(X_j,X_j)\tau} \nonumber \\
            &= \left[\left(\NTK(X,X_j) -\NNGP(X,X_j)\right)- \NNGP(X,X')\NNGP^{-1}(X',X')\left(I - e^{-\lambda \NNGP(X',X')\hat\tau}\right)\left(\NTK(X',X_j) - \NNGP(X',X_j)\right)\right] e^{-\lambda \NNGP(X_j,X_j)\tau} \nonumber \\
           &= \left[\left(\NTK(X,X_j) -\NNGP(X,X_j)\right)- \NNGP(X,X')T^{\lambda,\hat\tau}_{\NNGP}(X')\left(\NTK(X',X_j) - \NNGP(X',X_j)\right)\right] e^{-\lambda \NNGP(X_j,X_j)\tau}
            \label{eq:anil-kernel-test:supp}
        \end{align}
        where we used the equivalence 
        \begin{align}
            \nabla_{\theta_0^\body} f_{\thetatest_0}(\cdot) \cdot \nabla_{\theta_0^\body} f_{\theta_0}(\ast)^\top = \NTK(\cdot,\ast) - \NNGP(\cdot,\ast)
        \end{align}
        in the infinite width limit at initialization.
    \end{itemize}
    \item \textbf{MTL}
    \begin{itemize}[leftmargin=*,align=left,noitemsep,nolistsep]
        \item $\metaNTK_\mtl(\X,\X)=\nabla_{\thetamtl_0} \fmtl_{\thetamtl_0}(\X) \cdot  \nabla_{\thetamtl_0} \fmtl_{\thetamtl_0}(\X)^\top$. Notice that for any input with head index $i$, we have
        \begin{align*}
            \nabla_{\thetamtl_0} \fmtl_{\thetamtl_0}(\cdot,i) &= \nabla_{\thetamtl_0^\body} \fmtl_{\thetamtl_0}(\cdot,i) + \sum_{j=1}^{N+1} \nabla_{\wmtl{j}} \fmtl_{\thetamtl_0}(\cdot,i) \\
            &=\nabla_{\thetamtl_0^\body} \fmtl_{\thetamtl_0}(\cdot,i) + \nabla_{\wmtl{i}} \fmtl_{\thetamtl_0}(\cdot,i) \eq\label{eq:fmtl_0-grad}
        \end{align*}
        since for $j\neq i$, we have $\nabla_{\wmtl{j}} \fmtl_{\thetamtl_0}(x,i)=0$ based on the multi-head structure.
        
        Thus, we can write down the $(i,j)$-th block of $\metaNTK_\mtl(\X,\X)$ as
        \begin{align*}
            \qquad [\metaNTK_\mtl(\X,\X)]_{ij}  
            &= \nabla_{\thetamtl_0} \fmtl_{\thetamtl_0}(X_i,i)\nabla_{\thetamtl_0} \fmtl_{\thetamtl_0}(X_j,j)^\top  \\
            &= \nabla_{\thetamtl_0^\body}\fmtl_{\thetamtl_0}(X_i,i)\nabla_{\thetamtl_0^\body} \fmtl_{\thetamtl_0}(X_j,j)^\top 
            + \nabla_{\wmtl{i}}\fmtl_{\thetamtl_0}(X_i,i)\nabla_{\wmtl{j}} \fmtl_{\thetamtl_0}(X_j,j)^\top
        \end{align*}
        Note that for $i\neq j$, we have $\nabla_{\wmtl{i}}\fmtl_{\thetamtl_0}(X_i,i)\nabla_{\wmtl{j}} \fmtl_{\thetamtl_0}(X_j,j)^\top = 0$, since $\wmtl{i}$ and $\wmtl{j}$ are in different dimensions of $\thetamtl$. Thus,
        \begin{itemize}
            \item as $i\neq j$, we have\footnote{The following equivalence can be straightforwardly dervied based on Appendix D and E of \cite{lee2019wide}.}
        \begin{align*}
            [\metaNTK_\mtl(\X,\X)]_{ij}  
            = \nabla_{\thetamtl_0^\body}\fmtl_{\thetamtl_0}(X_i,i)\nabla_{\thetamtl_0^\body} \fmtl_{\thetamtl_0}(X_j,j)^\top 
            = \NTK(X_i,X_j) - \NNGP(X_i,X_j)
        \end{align*}
            \item as $i = j$, we have
            \begin{align*}
                [\metaNTK_\mtl(\X,\X)]_{ii} =  \nabla_{\thetamtl_0^\body}\fmtl_{\thetamtl_0}(X_i,i)\nabla_{\thetamtl_0^\body} \fmtl_{\thetamtl_0}(X_i,i)^\top 
                + \nabla_{\wmtl{i}}\fmtl_{\thetamtl_0}(X_i,i)\nabla_{\wmtl{i}} \fmtl_{\thetamtl_0}(X_i,i)^\top 
                = \NTK(X_i,X_i)
            \end{align*}
        \end{itemize}
        In conclusion, for $i,j\in[N]$, we have
        \begin{align*}
            [\metaNTK_\mtl(\X,\X)]_{ij}  
            = \NTK(X_i,X_j) - \indicator[i\neq j] \NNGP(X_i,X_j)
        \end{align*}
        Thus, 
        \begin{align}\label{eq:mtl-kernel:supp}
            \metaNTK_\mtl(\X,\X) = \NTK(\X,\X) - \NNGP(\X,\X) + \diag{\{\NNGP(X_i,X_i)\}_{i=1}^N}
        \end{align}

    \end{itemize}
    \item $\metaNTK_\mtl'((\xx,\hat\tau),\X)=\nabla_{\thetamtl_0^\body} \Fmtl^{\hat \tau}_{\thetamtl_0} (\xxy) \nabla_{\thetamtl_0^\body} \fmtl_{\thetamtl_0}(\X)^\top$. 

    Based on \eqref{eq:fmtl_0-grad}, following \eqref{eq:anil-kernel-test:supp}, we can express the $(1,j)$-th block of $\metaNTK_\mtl'((\xx,\hat\tau),\X)$ as
    \begin{align*}
        &\qquad\left[\metaNTK_\mtl'((\xx,\hat\tau),\X)\right]_{1j} \\
        &=  \nabla_{\thetamtl_0^\body}F^{\hat\tau}_{\thetamtltest_0}(\xxy) \nabla_{\thetamtl_0^\body} \fmtl_{\thetamtl_0}(\X)^\top\\
        &=\left(\nabla_{\thetamtl_0^\body} f_{\thetamtltest}(X) - \NNGP(X,X')\NNGP^{-1}(X',X')\left(I - e^{-\lambda \NNGP(X',X')\hat\tau}\right) \nabla_{\thetamtl_0^\body} f_{\thetamtltest}(X') \right)\cdot
         \nabla_{\thetamtl_0^\body} \fmtl_{\thetamtl_0}(X_j,j)^\top\\
        &=\nabla_{\thetamtl_0^\body} f_{\thetamtltest}(X) \nabla_{\thetamtl_0^\body} \fmtl_{\thetamtl_0}(X_j,j)^\top- \NNGP(X,X')\NNGP^{-1}(X',X')\left(I - e^{-\lambda \NNGP(X',X')\hat\tau}\right) \nabla_{\thetamtl_0^\body} f_{\thetamtltest}(X') \nabla_{\thetamtl_0^\body} \fmtl_{\thetamtl_0}(X_j,j)^\top\\
        &= \Big[\NTK(X,X_j) - \NNGP(X,X_j)\Big] - \NNGP(X,X')\NNGP^{-1}(X',X')\left(I - e^{-\lambda \NNGP(X',X')\hat\tau}\right) \Big[\NTK(X',X_j) - \NNGP(X',X_j)\Big]\\
        & = \NTK(X,X_j) - \NNGP(X,X')\T^{\hat\tau}_{\NNGP}(X')\NTK(X',X_j) -  \NNGP(X,X')\NNGP^{-1}(X',X') e^{-\lambda \NNGP(X',X')\hat\tau} \NNGP(X',X_j)\eq \label{eq:mtl-kernel-test:supp}
    \end{align*}
\end{itemize}
\end{proof}
\textbf{Remarks.} Notice that \eqref{eq:anil-kernel-test:supp} and \eqref{eq:mtl-kernel-test:supp} indicate the following relation:
    \begin{align*}
        \left[\metaNTK_\anil'((\xx,\hat\tau),\X)\right]_{1j} = \left[\metaNTK_\mtl'((\xx,\hat\tau),\X)\right]_{1j} e^{-\lambda \NNGP(X_j,X_j)\tau}
    \end{align*}
    Furthermore, it is straightforward to show that
    \begin{align}\label{eq:anil-mtl-test-kernel-relation}
        \metaNTK_\anil'((\xx,\hat\tau),\X) = \metaNTK_\mtl'((\xx,\hat\tau),\X) \cdot \diag{\{e^{-\lambda \NNGP(X_j,X_j)\tau}\}_{j=1}^N}
    \end{align}
    where $\diag{\{e^{-\lambda \NNGP(X_j,X_j)\tau}\}_{j=1}^N}$ is a diagonal block matrix with the $j$-th block as $e^{-\lambda \NNGP(X_j,X_j)\tau}$.
\subsubsection{Proof of Lemma \ref{lemma:test-predict}} \label{supp:proof:proof-of-lemma-test-predict}
Now, we can prove Lemma \ref{lemma:test-predict} shown in Sec. \ref{sec:mtl=meta:functional}, by leveraging Lemma \ref{lemma:lienarization} and Lemma \ref{lemma:anil-mtl-kernels} that we just proved. In particular, without loss of generality, following \citet{CNTK}, we assume the outputs of randomly initialized networks have a much smaller magnitude compared with the magnitude of training labels such that $\|f_{\theta_0}(x)\|_2 \leq \|y\|_2 \leq \cO(\width^{-\frac 1 2})$. Notice this can be always achieved by choosing smaller initialization scale or scaling down the neural net output \cite{CNTK}, without any effect on the training dynamics and the predictions, up to a width-dependent factor on the learning rate. Below, we present the steps of the proof in detail.

\begin{proof}[Proof of Lemma \ref{lemma:test-predict}]
Plugging the kernels expressions derived by Lemma \ref{lemma:anil-mtl-kernels} into \eqref{eq:F_t-test} and \eqref{eq:F_t-test:mtl}, and combining with the fact that $\lim_{\width \rightarrow \infty} \nngp_{w_0}\rightarrow \NNGP$ (proved by Corollary 1 of \citet{lee2019wide}), we obtain the expressions of \eqref{eq:lemma:test-pred:F_anil} and \eqref{eq:lemma:test-pred:F_mtl} in Lemma \ref{lemma:test-predict} in the infinite width limit. Notice that we consider sufficiently large width $\width$, then the discrepancy between the infinite-width kernels and their finite-width counter-parts (i.e., the finite-width correction) is bounded by $\cO(\frac{1}{\sqrt{\width}})$ with arbitrarily large probability, indicated by Theorem 1 of \citet{Hanin2020Finite}. Thus, the finite-width correction terms are absorbed into the $\cO(\frac{1}{\sqrt{\width}})$ error terms in \eqref{eq:F_t-test} and \eqref{eq:F_t-test:mtl}.

\end{proof}
\subsubsection{Discrepancy between Predictions of ANIL and MTL}\label{supp:proof:pred-diff}
Based on \eqref{eq:anil-kernel:supp}, \eqref{eq:mtl-kernel:supp}, and \eqref{eq:anil-mtl-test-kernel-relation}, for small $\lambda \tau$, the discrepancy between ANIL and MTL predictions can be written as (Note: we consider neural nets trained under ANIL and MTL for infinite time $t=\infty$, then take their parameters $\theta_\infty$ and $\thetamtl_\infty$ for test on any task $\task=(\xyxy)$),
\begin{align*}
    & \qquad F_\anil(\xxy) - F_\mtl(\xxy) \\
    &= F_{\thetatest_{\infty}}(\xxy) - F_{\thetamtltest_{\infty}}(\xxy)\\
    &= \left[ \metaNTK_\anil'((\xx,\hat\tau),\X) \metaNTK_\anil^{-1}(\X,\X)-\metaNTK_\mtl'((\xx,\hat\tau),\X) \metaNTK_\mtl^{-1}(\X,\X) \right]\Y \\
    &\quad - \metaNTK_\anil'((\xx,\hat\tau),\X) \metaNTK_\anil^{-1}(\X,\X)\overbrace{G_\tau(\XXY)}^{ = \cO(\lambda \tau \Lev(\NNGP)) } +\cO(\frac{1}{\sqrt h})\\
    &= \Big[\metaNTK_\mtl'((\xx,\hat\tau),\X) \cdot \diag{\{e^{-\lambda \NNGP(X_j,X_j)\tau}\}_{j=1}^N} \diag{\{e^{-\lambda \NNGP(X_i,X_i)\tau}\}_{i=1}^N}^{-1}\NTK(\X,\X)^{-1} \diag{\{e^{-\lambda \NNGP(X_i,X_i)\tau}\}_{i=1}^N}^{-1} \\
    & \quad - \metaNTK_\mtl'((\xx,\hat\tau),\X) \metaNTK_\mtl^{-1}(\X,\X)\Big]\Y + \cO(\lambda \tau)+\cO(\frac{1}{\sqrt h})\\
    &=\metaNTK_\mtl'((\xx,\hat\tau),\X)\left[\NTK(\X,\X)^{-1} \diag{\{e^{-\lambda \NNGP(X_i,X_i)\tau}\}_{i=1}^N}^{-1} - \metaNTK_\mtl^{-1}(\X,\X)\right]\Y + \cO(\lambda \tau)+\cO(\frac{1}{\sqrt h})\\
    &= \metaNTK_\mtl'((\xx,\hat\tau),\X)\Big[\NTK(\X,\X)^{-1} \underbrace{\diag{\{e^{-\lambda \NNGP(X_i,X_i)\tau}\}_{i=1}^N}^{-1}}_{= I + \cO(\lambda \tau \Lev(\NNGP))} - \metaNTK_\mtl^{-1}(\X,\X)\Big]\Y + \cO(\lambda \tau \Lev(\NNGP))+\cO(\frac{1}{\sqrt h})\\
    &= \metaNTK_\mtl'((\xx,\hat\tau),\X)\Big[\NTK(\X,\X)^{-1} - \metaNTK_\mtl^{-1}(\X,\X)\Big]\Y + \cO(\lambda \tau \Lev(\NNGP))+\cO(\frac{1}{\sqrt h})\eq\label{eq:anil-mtl-diff:general}
\end{align*}
where $\Lev(\NNGP)\triangleq \max\left(\{\Lev(\NNGP(X_i,X_i))\}_{i=1}^N\right)$.

\textbf{Remarks.} \eqref{eq:anil-mtl-diff:general} indicates that for small $\lambda \tau$, the discrepancy between ANIL's and MTL's test predictions is determined by 
\begin{align}\label{eq:ntk-mtl-inv-diff}
\NTK(\X,\X)^{-1} - \metaNTK_\mtl^{-1}(\X,\X).    
\end{align}
Thus, if this difference vanishes in some limit, ANIL and MTL will output almost the same predictions on any test task.

\subsection{Kernel Structures for Deep ReLU Nets}\label{supp:proof:relu-kernel-structures}

\textbf{Setup.} As described by Sec. \ref{sec:mtl=meta:functional}, we focus on networks that adopt ReLU activation and He's initialization, and we consider the inputs are normalized to have unit variance, without loss of generality. Besides, we also assume any pair of samples in the training set are distinct.

\textbf{NTK and NNGP Kernel Structures.} \citet{xiao2020dis} shows that for ReLU networks with He's initialization and unit-variance inputs, the corresponding NTK and NNGP kernels have some special structures. Specifically, at large depth, the spectra of these kernels can be characterized explicitly, as shown by Lemma \ref{lemma:ntk-nngp-large-depth} below, which is adopted and rephrased from the Appendix C.1 of \citet{xiao2020dis}.
\begin{lemma}[Kernel Structures of NTK and NNGP]\label{lemma:ntk-nngp-large-depth} For sufficiently large depth $L$, NTK and NNGP kernels have the following expressions\footnote{Notice that we use the little-o notation here: $f(x) = o(g(x))$ indicates that $g(x)$ grows much faster than $f(x)$. Thus the $o(\cdot)$ terms are negligible here.} (Note: we use the superscript $^{(L)}$ to mark the kernels' dependence on the depth $L$)
    \begin{align}
        \NTK^{(L)}(\X,\X) &= L\left(\frac{1}{4} \ones{Nn}{Nn} + \frac{3}{4} I \right)+ \bA^{(L)}_{\X,\X} \label{eq:large-depth:ntk}\\
        \NNGP^{(L)}(\X,\X) &= \ones{Nn}{Nn} + \frac{1}{L^2} \bB^{(L)}_{\X,\X}\label{eq:large-depth:nngp}
    \end{align}
where $\bA^{(L)}_{\X,\X}, \bB^{(L)}_{\X,\X} \in \bR^{Nn \times Nn}$ is a symmetric matrix with elements of $\cO(1)$.

The eigenvalues of $\NTK^{(L)}(\X,\X)$ and $\NNGP^{(L)}(\X,\X)$ are all positive since $\NTK$ and $\NNGP$ are guaranteed to be positive definite, and these eigenvalues can be characterized as
\begin{align}
    \begin{cases}
    \Lev(\NTK(\X,\X)) = \frac{Nn+3}{4} L + \cO(1) \\
    \bulkev(\NTK(\X,\X)) = \frac{3}{4} L + \cO(1) 
    \end{cases}
    \qquad 
    \begin{cases}
        \Lev(\NNGP(\X,\X)) = Nn + \cO(\frac 1 L^2) \\
        \bulkev(\NNGP(\X,\X)) = \cO(\frac{1}{L^2}) 
        \end{cases}
\end{align}
where $\bulkev(\cdot)$ denotes the eigenvalues besides the largest eigenvalue.
\end{lemma}

\textbf{Discrepancy between Kernel Inverses.} As shown by Appendix \ref{supp:proof:pred-diff}, the discrepancy between the predictions of ANIL and MTL is controlled by \eqref{eq:ntk-mtl-inv-diff}, i.e., $\NTK^\inv(\X,\X) - \metaNTK_\mtl^\inv(\X,\X)$. In the lemma below, we study \eqref{eq:ntk-mtl-inv-diff} in the setting of ReLU nets with He's initialization, and prove a bound over the operator norm of \eqref{eq:ntk-mtl-inv-diff}.

\begin{lemma}[Discrepancy between Kernel Inverses]\label{lemma:kernel-inv-diff}
There exists $L^*\in \mathbb{N}^{+}$ s.t. for $L\geq L^*$, 

\begin{align}\label{eq:med-depth:evals-condition}
    \begin{cases}
    \Lev\left(\NTK^{(L)}(\X,\X)\right) &\simeq \cO(NnL)  \gg \eval_2\left(\NTK^{(L)}(\X,\X)\right)\\
    \frac{1}{Nn}\bone{Nn}^\top \NTK^{(L)}(\X,\X)\bone{Nn} &\simeq \cO(NnL) \gg \eval_2\left(\NTK^{(L)}(\X,\X)\right)\\
    \Lev\left(\NTK^{(L)}(\X,\X)\right) &\geq \cO(L) \cdot \Lev \left(\NNGP^{(L)}(\X,\X)\right)\\
    \end{cases}
\end{align}
where $\eval_2(\cdot)$ denotes the second largest eigenvalue. Then, we have
\begin{align}\label{eq:ntk-mtl:F-norm:supp}
    \|\NTK(\X,\X)^{-1} - \metaNTK_\mtl^{-1}(\X,\X)\|_{\op} \leq \cO(\frac {1}{L^2})
\end{align}
    
\end{lemma}

\begin{proof}
From \eqref{eq:mtl-kernel:supp}, we know (Note: we omit the superscript $^{(L)}$ for simplicity in this proof)
\begin{align*}
    \metaNTK_\mtl(\X,\X) & = \NTK(\X,\X) - \NNGP(\X,\X) + \diag{\{\NNGP(X_i,X_i)\}_{i=1}^N} \\
    & = \NTK(\X,\X) - \wdNNGP(\X,\X)
\end{align*}
where we denote $\wdNNGP (\X,\X) = \NNGP (\X,\X) + \diag{\{\NNGP(X_i,X_i)\}_{i=1}^N}$ for simplicity.

\underline{\textbf{Case I:} $n=1$.}

In this case, obviously, for each $i\in[N]$, we have $\NNGP(X_i,X_i) = 1 + \cO(\frac 1 L^2) \in \bR$. We can define a perturbed NNGP matrix as
\begin{align}\label{eq:perturbed-nngp}
    \wdNNGP(\X,\X) &= \NNGP(\X,\X) - \diag{\{\NNGP(X_i,X_i)\}_{i=1}^N}\\
     & = \ones{N}{N} - I + \frac{1}{L^2}  \widetilde{\bB}^{(L)}_{\X,\X} 
\end{align}
where we define $\widetilde{\bB}^{(L)}_{\X,\X}  = \bB^{(L)}_{\X,\X} - \left(\diag{\{\NNGP(X_i,X_i)\}_{i=1}^N} - I\right)$, i.e., $\bB^{(L)}_{\X,\X}$ with the $\cO(\frac 1 L^2)$ terms from $\diag{\{\NNGP(X_i,X_i)\}_{i=1}^N}$.

For convenience, let us define a perturbed NTK matrix as 
\begin{align}\label{eq:perturbed-ntk}
    \wdNTK(\X,\X) &= \NTK(\X,\X) - \left(\wdNNGP(\X,\X) - \ones{N}{N}\right)\nonumber\\
    &=\NTK(\X,\X) + I - \frac{1}{L^2}  \widetilde{\bB}^{(L)}_{\X,\X}.
\end{align}
Obviously, we have
\begin{align*}
    \|\NTK(\X,\X)^{-1} - \metaNTK_\mtl^{-1}(\X,\X)\|_{\op} &= \|\NTK(\X,\X)^{-1} - \wdNTK^{-1}(\X,\X) + \wdNTK^{-1}(\X,\X) - \metaNTK_\mtl^{-1}(\X,\X)\|_{\op}\\
    &\leq \|\NTK(\X,\X)^{-1} - \wdNTK^{-1}(\X,\X)\|_{\op} + \|\wdNTK^{-1}(\X,\X) - \metaNTK_\mtl^{-1}(\X,\X)\|_{\op}\eq\label{eq:ntk-mtl-inv:norm:bound}
\end{align*}
Thus, we can prove \eqref{eq:ntk-mtl:F-norm:supp} by providing bounds for $\|\NTK(\X,\X)^{-1} - \wdNTK^{-1}(\X,\X)\|_{\op}$ and $\|\wdNTK^{-1}(\X,\X) - \NTK^{-1}(\X,\X)\|_{\op}$ separately.
\begin{itemize}
    \item \textbf{Bound $\|\metaNTK_\mtl^{-1} - \wdNTK^{-1}(\X,\X)\|_{\op}$.}

By the Woodbury identity, we have
\begin{align*}
    \metaNTK_\mtl^{-1}(\X,\X) 
    &= \left(\NTK(\X,\X) - \wdNNGP(\X,\X)\right)^{-1}\\
    &= \Big(\big[\overbrace{\NTK(\X,\X) + I - \frac{1}{L^2}  \widetilde{\bB}^{(L)}_{\X,\X} - o(\frac{1}{L^2})}^{\wdNTK(\X,\X) \triangleq }\big] - \ones{N}{N}\big)^{-1}\\
    &= \left(\wdNTK(\X,\X) - \ones{N}{N}\right)^{-1}\\
    &= \wdNTK(\X,\X)^{-1} - \rho \cdot \wdNTK(\X,\X)^{-1} \ones{N}{N} \wdNTK(\X,\X)^{-1}
\end{align*}
where $$\rho = \frac{1}{1 - \bone{N}^\top \wdNTK(\X,\X)^{-1} \bone{N} }$$

By \eqref{eq:med-depth:evals-condition} and some eigendecomposition analysis, we can easily derive that 
\begin{align*}
    \rho = \frac{1}{1 - \bone{N}^\top \wdNTK(\X,\X)^{-1} \bone{N} } &\simeq \frac{1}{1 - \cO(\frac{1}{L})}\\
    \wdNTK(\X,\X)^{-1} \ones{N}{N} \wdNTK(\X,\X)^{-1}&\simeq \cO\left(\frac{1}{N^2L^2}\right)\ones{N}{N} 
\end{align*}
Thus 
\begin{align}
    \metaNTK_\mtl^{-1}(\X,\X) = \wdNTK(\X,\X)^{-1} - \cO\left(\frac{1}{N^2L^2(1 - \cO(\frac{1}{L}))}\right)\ones{N}{N} 
\end{align}
where the last term is negligible since its \textit{maximum} eigenvalue is $\cO(\frac{1}{NL^2(1 - \cO(\frac{1}{L})}))$, while the \textit{minimum} eigenvalue for the first term is $\cO(\frac{1}{NL})$.

Thus, we can write
\begin{align}\label{eq:mtl-pertrbed-ntk:norm:bound}
    \|\metaNTK_\mtl^{-1}(\X,\X) - \wdNTK(\X,\X)^{-1}\|_{\op} = \|\cO\left(\frac{1}{N^2L^2(1 - \cO(\frac{1}{L}))}\right)\ones{N}{N} \|_{\op} \leq \cO(\frac{1}{NL^2})
\end{align}

\item Bound $\|\wdNTK^{-1}(\X,\X) - \NTK^{-1}(\X,\X)\|_{\op}$

By \eqref{eq:large-depth:ntk}, \eqref{eq:perturbed-ntk}, we know
\begin{align*}
    \wdNTK(\X,\X) &= \Big(\overbrace{ \frac{L}{4} \ones{N}{Nn} + \frac{3L}{4} I + \bA^{(L)}_{\X,\X} }^{\NTK(\X,\X)}\Big) - \Big(\overbrace{\ones{N}{N} - I + \frac{1}{L^2}  \widetilde{\bB}^{(L)}_{\X,\X}}^{\wdNNGP(\X,\X) - \ones{N}{N}}\Big) \\
    &= \left(\frac{L}{4} - 1\right)\ones{N}{N} + \left(\frac{3L}{4} + 1\right)I + \left(\bA^{(L)}_{\X,\X} - \frac{1}{L^2}  \widetilde{\bB}^{(L)}_{\X,\X}\right)
\end{align*}
By observation, it is obvious that for relatively large $L$, the perturbation $\ones{N}{N} - I + \frac{1}{L^2}  \widetilde{\bB}^{(L)}_{\X,\X}$ has minimal effect, e.g., the spectrum of $\wdNTK(\X,\X)$ is almost identical to $\NTK(\X,\X)$.

Now, let us bound the inverse of the perturbed matrix $\wdNTK(\X,\X)$ formally. 

Leveraging the identity $(A+B)^\inv = A^\inv A^\inv B (A+B)^\inv$ from \cite{henderson1981deriving}. Defining $$\hDelta = \wdNTK(\X,\X) - \NTK(\X,\X)=\ones{N}{N} - I + \frac{1}{L^2} \widetilde{\bB}^{(L)}_{\X,\X}$$
then we have
\begin{align*}
  &\qquad \left\|\wdNTK(\X,\X)^{-1} - \NTK(\X,\X)^\inv \right\|_{op}\\
  &= \left\|\left(\NTK(\X,\X) + \hDelta\right)^{-1}\right\|_{op}\\
  &=\left\|\NTK(\X,\X)^\inv + \NTK(\X,\X)^\inv \hDelta \left(\NTK(\X,\X) + \Delta \right)^\inv -\NTK(\X,\X)^\inv \right\|_{op}\\
  & = \left\|\NTK(\X,\X)^\inv \hDelta \wdNTK(\X,\X)^\inv \right\|_{op}\\
  & = \left\|\NTK(\X,\X)^\inv \left(\ones{N}{N} - I + \frac{1}{L^2} \widetilde{\bB}^{(L)}_{\X,\X}\right) \wdNTK(\X,\X)^\inv \right\|_{op} \\
  &\leq \left\| \wdNTK(\X,\X)^{-1}\ones{N}{N} \wdNTK(\X,\X)^{-1} \right\|_\op 
    + \left\| \NTK(\X,\X)^{-1}I\wdNTK(\X,\X)^{-1} \right\|_\op
    + \frac 1 L^2 \left\| \wdNTK(\X,\X)^{-1}\widetilde{\bB}^{(L)}_{\X,\X} \wdNTK(\X,\X)^{-1} \right\|_\op\eq\label{eq:perturbed-NTK:norm:bound:leq}\\
    &\leq \cO\left(\frac{1}{NL^2}\right) +\cO\left(\frac{1}{L^2}\right) + \cO\left(\frac{1}{L^4}\right) \\
    &\leq \cO\left(\frac{1}{L^2}\right)\eq\label{eq:perturbed-NTK:norm:bound}
\end{align*}
\end{itemize}

Finally, combining \eqref{eq:ntk-mtl-inv:norm:bound}, \eqref{eq:mtl-pertrbed-ntk:norm:bound} and \eqref{eq:perturbed-NTK:norm:bound}, we have 
\begin{align*}
 \|\NTK(\X,\X)^{-1} - \metaNTK_\mtl^{-1}(\X,\X)\|_{\op}
    &\leq \|\NTK(\X,\X)^{-1} - \wdNTK^{-1}(\X,\X)\|_{\op} + \|\wdNTK^{-1}(\X,\X) - \metaNTK_\mtl^{-1}(\X,\X)\|_{\op}    \\
    &\leq \cO(\frac{1}{NL^2}) + \cO(\frac{1}{L^2}) = \cO(\frac{1}{L^2})\label{eq:n=1:final-bound}\eq
\end{align*}

\underline{\textbf{Case II:} $n>1$.}

Compared to the case of $n=1$, the only difference with \eqref{eq:n=1:final-bound} is caused by the term $\left\| \NTK(\X,\X)^{-1}I\wdNTK(\X,\X)^{-1} \right\|_\op$ in \eqref{eq:perturbed-NTK:norm:bound:leq} is converted to 
$$\left\|\NTK(\X,\X)^{-1}\diag{\{\ones{n}{n}\}_{i=1}^N}\wdNTK(\X,\X)^{-1} \right\|_\op$$
Since $$\|\diag{\{\ones{n}{n}\}_{i=1}^N}\|_{op} = \|\ones{n}{n}\|_{op}=  n = \cO(1)~,$$
we have
\begin{align*}
    \left\| \NTK(\X,\X)^{-1}\diag{\{\ones{n}{n}\}_{i=1}^N}\wdNTK(\X,\X)^{-1} \right\|_\op &\leq \left\|\wdNTK(\X,\X)^{-1}\right\|_\op^2
    \left\|\diag{\{\ones{n}{n}\}_{i=1}^N}\right\|_\op\\
    &\leq \cO(\frac{1}{L^2})
\end{align*}
\end{proof}

\subsection{Proof of Theorem \ref{thm:closeness}}\label{supp:proof:main-theorem}

The proof of Theorem \ref{thm:closeness} can be straightforwardly derived based on Lemma \ref{lemma:kernel-inv-diff}.
\begin{proof}
By \eqref{eq:anil-mtl-diff:general}, \eqref{eq:ntk-mtl:F-norm:supp}, we have
\begin{align*}
    &\qquad \|F_\anil(\xxy) - F_\mtl(\xxy)\|_2 \\
    &\leq \|\metaNTK_\mtl'((\xx,\hat\tau),\X)\|_\op \|\NTK(\X,\X)^{-1} - \metaNTK_\mtl^{-1}(\X,\X)\|_\op \|\Y\|_2 + \cO(\lambda \tau \Lev(\NNGP)) +\cO(\frac{1}{\sqrt h})\\
    &\leq \cO(\frac{1}{L}) + \cO(\lambda \tau) +\cO(\frac{1}{\sqrt h})
\end{align*}
where we used the facts that $\|\metaNTK_\mtl'((\xx,\hat\tau),\X)\|_\op = 
\cO(L)$, which can be straightforwardly derived from Lemma \ref{lemma:anil-mtl-kernels} and \ref{lemma:ntk-nngp-large-depth}.
\end{proof}

\subsection{Extension to Residual ReLU Networks}\label{supp:proof:residual}

Corollary \ref{corollary:resnets} states that the theoretical results of Theorem \ref{thm:closeness} apply to residual ReLU networks and residual ReLU networks with LayerNorm. The proof of this corollary is simply derived from Appendix C.2 and C.4 of \citet{xiao2020dis}. 
\begin{proof}
For residual ReLU networks, the corresponding NTK and NNGP have a factor of $e^{L}$ compared \eqref{eq:large-depth:ntk} and \eqref{eq:large-depth:nngp}, which has no effect on the predictors $F_\anil$ and $F_\mtl$, since the factors from the kernel and kernel inverse cancel out (e.g., $e^{L}\metaNTK_\mtl'((X,X',\hat\tau),\X) \cdot (e^{L} \metaNTK_\mtl(\X,\X))^{-1} = \metaNTK_\mtl'((X,X',\hat\tau),\X)  \metaNTK_\mtl(\X,\X)^{-1}$). Thus, Theorem \ref{thm:closeness} applies to this class of networks.

For residual ReLU networks with LayerNorm, Appendix C.3 of \citet{xiao2020dis} shows the kernel structures of NTK and NNGP is the same as ReLU networks without residual connections. Thus, Theorem \ref{thm:closeness} directly applies to this class of networks.
\end{proof}

\section{Details of Experiments}\label{supp:exp}
In this section, we will provide more details about the experiment in Sec. \ref{sec:exp}. Specifically,
\begin{itemize}
    \item Appendix \ref{supp:exp:theory-validate}: presents more experimental details about Sec. \ref{sec:exp:theory-validate}, the empirical validation of Theorem \ref{thm:closeness}.
    \item Appendix \ref{supp:exp:few-shot}: presents more experimental details about Sec. \ref{sec:exp:few-shot}, the empirical study on few-shot image classification benchmarks.
\end{itemize}
\subsection{Empirical Validation of Theorem \ref{thm:closeness}}\label{supp:exp:theory-validate}

\textbf{Implementation.} We implement MTL and ANIL kernels with Neural Tangents \cite{neuraltangents2020}, a codebase built on JAX \cite{jax2018github}, which is a package designed for high-performance machine learning research in Python. Since MTL and ANIL kernel functions are composite kernel functions built upon NTK and NNGP functions, we directly construct NTKs and NNGPs using Neural Tangents and then compose them into MTL and ANIL kernels.

\textbf{About Figure \ref{fig:thm-validate}.} Note that the value at $L=10$ in the first image is a little smaller than the value at $\lambda\tau = 0$ in the second image. That is because the random seeds using in the two images are different. Even though we take an average over 5 random seeds when plotting each image, there still exists some non-negligible variance. 

\subsection{Experiments on Few-Shot Image Classification Benchmarks} \label{supp:exp:few-shot}


\textbf{Fine-Tuning in Validation and Test.} In the meta-validation and meta-testing stages, following Sec. \ref{sec:prelim:fine-tune}, we fine-tune a linear classifier on the features (i.e., outputs of the last hidden layer) with the cross-entropy loss and a $\ell_2$ regularization. Specifically, similar to \citet{tian2020rethink}, we use the logistic regression classifier from sklearn for the fine-tuning \cite{sklearn}, and we set the $\ell_2$ regularization strength to be $0.33$ based on the following ablation study on $\ell_2$ penalty (i.e., Table \ref{tab:l2}.
\begin{table}[b]
    \centering
    \setlength\tabcolsep{1.7pt}
    \begin{tabular}{c c c c c c c c c}
        {\small $\ell_2$ Penalty}  &\small $0.0001$& \small{$0.001$}& \small{0.01}& \small{0.1} & \small{0.33} & \small{1} & \small{3}\\
        \midrule
        {\small Test Accuracy(\%)} & \small 76.86 & \small 77.02 & \small 77.28 & \small 77.61 &\small \textbf{77.72} & \small 77.55 &\small 76.82\\
    \end{tabular}
    \caption{Ablation study of the $\ell_2$ penalty on the fine-tuned linear layer. Evaluated on mini-ImageNet (5-way 5-shot classification).
    }\label{tab:l2}
\end{table}    

\nocite{finn2019online,provable-gbml,adaptive-GBML,hu2020biased,xu2020meta}
\nocite{maml_nonconvex,ji2020multistep,imaml,zhou2019metalearning}

\end{document}